\documentclass[sigconf,natbib=false]{acmart}

\def\shownotes{1}
\usepackage{macro}
\usepackage[backend=biber,
            style=numeric-comp,
            bibstyle=numeric,
            natbib=true,
            maxbibnames=99,
            minalphanames=3,
            maxcitenames=2,
            sorting=ynt,
            giveninits=true,            
            labeldateparts=true,
            sortcites=true,
            backref=true]{biblatex}
\renewcommand{\bibfont}{\bibliofont}
\DefineBibliographyStrings{english}{%
  backrefpage = {\hspace{-0.0500cm}},%
  backrefpages = {\hspace{-0.0500cm}},%
}
\usepackage{multirow, makecell}
\usepackage{graphicx}
\usepackage{booktabs} %
\usepackage{enumitem}
\usepackage{wrapfig}
\usepackage{subcaption}
\usepackage{balance}
\usepackage{algorithm,algorithmicx,algpseudocode}

\AtBeginDocument{%
  \providecommand\BibTeX{{%
    \normalfont B\kern-0.5em{\scshape i\kern-0.25em b}\kern-0.8em\TeX}}}

\copyrightyear{2024}
\acmYear{2024}
\setcopyright{rightsretained}
\acmConference[KDD '24]{Proceedings of the 30th ACM SIGKDD Conference on Knowledge Discovery and Data Mining}{August 25--29, 2024}{Barcelona, Spain}
\acmBooktitle{Proceedings of the 30th ACM SIGKDD Conference on Knowledge Discovery and Data Mining (KDD '24), August 25--29, 2024, Barcelona, Spain}\acmDOI{10.1145/3637528.3671835}
\acmISBN{979-8-4007-0490-1/24/08}

\makeatletter
\gdef\@copyrightpermission{
  \begin{minipage}{0.3\columnwidth}
   \href{https://creativecommons.org/licenses/by/4.0/}{\includegraphics[width=0.90\textwidth]{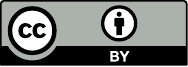}}
  \end{minipage}\hfill
  \begin{minipage}{0.7\columnwidth}
   \href{https://creativecommons.org/licenses/by/4.0/}{This work is licensed under a Creative Commons Attribution International 4.0 License.}
  \end{minipage}
  \vspace{5pt}
}
\makeatother

\begin{document}

\title{Scalable Multitask Learning Using Gradient-based Estimation of Task Affinity}

\author{Dongyue Li}
\affiliation{%
  \institution{Northeastern University}
  \city{Boston}
  \country{USA}
}
\email{li.dongyu@northeastern.edu}

\author{Aneesh Sharma}
\affiliation{%
  \institution{Google}
  \city{Mountain View}
  \country{USA}
}
\email{aneesh@google.com}

\author{Hongyang R. Zhang}
\affiliation{%
  \institution{Northeastern University}
  \city{Boston}
  \country{USA}
}
\email{ho.zhang@northeastern.edu}

\renewcommand{\shortauthors}{Dongyue Li, Aneesh Sharma, \& Hongyang R. Zhang}

\begin{abstract}
Multitask learning is a widely used paradigm for training models on diverse tasks, with applications ranging from graph neural networks to language model fine-tuning. Since tasks may interfere with each other, a key notion for modeling their relationships is {\em task affinity}. This includes pairwise task affinity, computed among pairs of tasks, and higher-order affinity, computed among subsets of tasks. Naively computing either of them requires repeatedly training on data from various task combinations, which is computationally intensive. We present a new algorithm \acronym{} that can estimate task affinities without this repeated training.

The key idea of \acronym{} is to train a ``base'' model for all tasks and then use a linearization technique to estimate the loss of the model for a specific task combination. The linearization works by computing a gradient-based approximation of the loss, using low-dimensional projections of gradients as features in a logistic regression to predict labels for the task combination. We show that the linearized model can provably approximate the loss when the gradient-based approximation is accurate, and also empirically verify that on several large models. Then, given the estimated task affinity, we design a semi-definite program for clustering similar tasks by maximizing the average density of clusters.

We evaluate \acronym's performance across seven datasets, including multi-label classification on graphs, and instruction fine-tuning of language models. Our task affinity estimates are within 2.7\% distance to the true affinities while needing only 3\% of FLOPs in full training. On our largest graph with 21M edges and 500 labeling tasks, our algorithm delivers estimates within 5\% distance to the true affinities, using only 112 GPU hours. Our results show that \acronym{} achieves excellent performance and runtime tradeoffs compared to existing approaches.
\end{abstract}




\maketitle
\section{Introduction}

Modern applications of neural networks often employ a single neural network for prediction or classification on multiple tasks. This multitask learning setup is widely used across a variety of settings, with examples such as a visual system that aims to detect various objects in autonomous driving simultaneously \cite{standley2020tasks}, a Graph Neural Network for community detection on large networks \cite{li2023boosting}, and prompt-tuning of pre-trained LLMs for NLP tasks \cite{longpre2023flan}. This multitask learning setup is not only computationally efficient (a single network can jointly predict many tasks), but it often improves prediction accuracy due to transfer learning.

The often implicit assumption behind multitask modeling is that there is a {\em positive} transfer effect among tasks \cite{ben2002theoretical}. However, as the number of tasks increases, one frequently observes a {\em negative transfer} effect in many applications, such as for prompt tuning of large language models, where adding a task to the model degrades performance on one or more tasks \cite{ wu2019understanding,wu2020generalization,yang2020precise,vu2021spot}. This observation has motivated a line of work that aims to group the tasks into subsets such that negative transfer among tasks within a subset is minimized, allowing one to train a separate multitask model per subset, thereby improving performance on all tasks \cite{li2023boosting}.

A key concept underlying many multitask learning algorithms is a notion of {\em task affinity}, which can capture the abovementioned positive or negative transfer effects across tasks precisely. For instance, one can compare pairwise task affinity \cite{standley2020tasks,fifty2021efficiently}---the loss of a model trained on each pair of tasks---against the loss of a model trained on each task. Given a notion of task affinity, a common recipe for designing multitask learning algorithms involves (1) {\em Task affinity computation} that builds a task affinity matrix, then (2) {\em task grouping} that uses this task affinity matrix to group tasks with positive transfers together, and finally (3) {\em multitask training} that fits a separate model per task group.

The performance improvement achieved through this paradigm depends on the task affinity and the grouping procedure. Moreover, the ability to leverage this paradigm hinges on the computation of task affinity (Step 1 above), which becomes expensive as the number of tasks grows. As a case in point, the computational complexity of pairwise task affinity scales quadratically with the number of tasks: this implies that even for community detection with 100 labelings, using pairwise task affinity requires training nearly 5000 models for computing the affinity matrix.

In this paper, we scale up this multitask learning paradigm by dramatically speeding up the first step of task affinity computation for two canonical examples of task affinities: Pairwise and higher-order task affinity (See Examples \ref{ex_pairwise}, \ref{ex_ho}). In our experiments on various real-world datasets representing different applications, our algorithm can reduce the task affinity computation time by nearly 32$\times$ compared to full model training while incurring less than $2.7\%$ error. In addition to this dramatic efficiency improvement, we also design a more robust method for task grouping (Step 2). Together, these new techniques match or improve the performance of previous multitask models.

\begin{figure}[!t]
    \centering
    \includegraphics[width=0.350\textwidth]{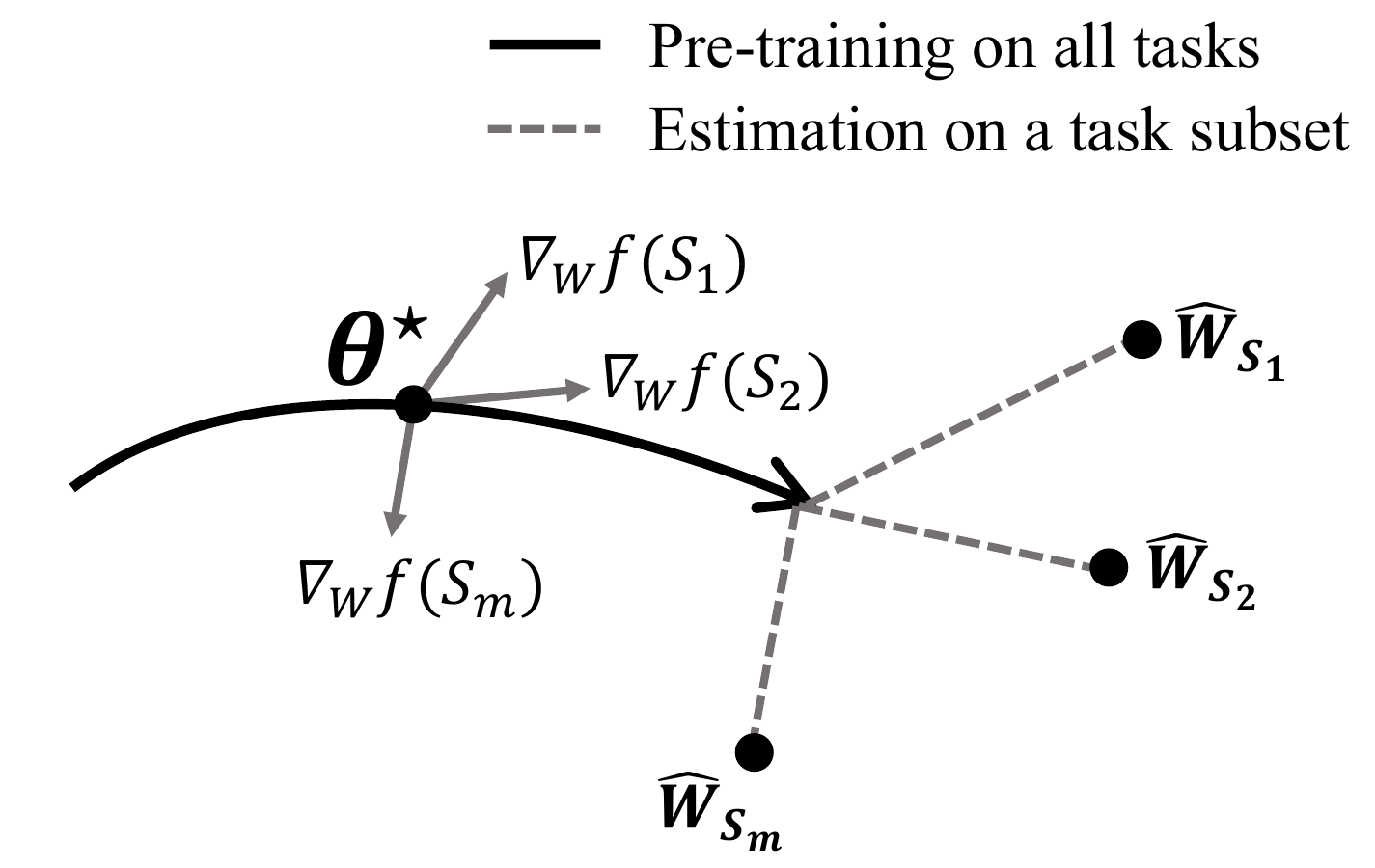}
    \vspace{-0.15in}
    \caption{Visualization of the gradient-based model approximation step in our \algo~algorithm, where we replace multitask training with a regression-based estimation of model parameters fine-tuned on a particular subset of tasks.}
    \label{fig_intro_fig}
    \vspace{-0.15in}
\end{figure}

The primary challenge for task affinity computation is avoiding training many multitask models on various task combinations. The key technical insight behind our algorithm is to leverage a {\em linearization} property of deep neural networks, including large language models. The linearization property for a neural network means we can approximate the model loss for a pre-trained meta-initialization and an input/output pair by using a gradient-based Taylor's expansion centered at the meta-initialization.
This linearization property has been observed for large language model fine-tuning in recent works, albeit not for the purpose of multitask learning \cite{malladi2023kernel,wei2023ntk,malladi2023fine}. Here, we leverage linearization to estimate task affinities in an efficient manner by using the first-order Taylor expansion from a pre-trained model, thereby saving the computation of backpropagation during model fine-tuning. This algorithm, \algo, is illustrated in Figure~\ref{fig_intro_fig}.

In more detail, we first compute the gradient at the initialization and then map the gradients to task labels with logistic regression.
The dimension of this regression can be high, especially for heavily parameterized models.
Thus, we use a dimension reduction technique and apply the Johnson-Lindenstrauss Lemma to give an error analysis.
On experiments of datasets with 100 tasks, we show that this approach estimates pairwise task affinity with \textbf{45}$\times$ fewer FLOPs and \textbf{11}$\times$ less GPU hours than fully computing the true scores, with only 5.7\% relative error. For higher-order task affinity, our approach uses \textbf{32}$\times$ fewer FLOPs and \textbf{5}$\times$ less GPU hours, with only 2.7\% relative error.
Furthermore,  our approach also scales to a large-scale graph with over 21M edges and 500 tasks. It estimates the task affinities within 5\% relative errors with 112.3 GPU hours, while computing the true affinity scores can take over 8000 GPU hours.  
Our algorithm is also suitable for accelerating task selection methods that are typically computationally expensive. An example is forward or backward subset selection \cite{hastie2009elements}, which is a popular heuristic but requires evaluating quadratically many task combinations.

As for the second step, we design a new clustering algorithm that uses these estimated task influences through a semi-definite programming (SDP) relaxation formulation. The clustering algorithm takes the estimated task affinity matrix $T$ (of size $n \times n$) \& the number $k$ of task groups as input, then solves an SDP for maximizing the average density of the $k$ groups. Since the SDP is a convex program, it can be solved efficiently, and we round the resulting solution to get the final task groups. Our experiments indicate that our clustering algorithm is more robust and performant than commonly used clustering techniques such as spectral clustering \cite{ng2001spectral} and Lloyd's algorithm \cite{lloyd1982least}. Once we have the task groups from the clustering, we can partition the tasks into subsets and train a separate model on tasks within each subset --- this overall algorithm is called \acronym.
Experiments show that our approach achieves the Pareto optimum regarding error rate and computation cost. For multi-label prediction on graphs trained with a $3$-layer GNN, \acronym~achieves comparable performance with over four baselines while using \textbf{32}$\times$ fewer FLOPs and \textbf{5}$\times$ less GPU hours. For instruction tuning of language models using T5-Base, \acronym~uses \textbf{48}$\times$ fewer FLOPs and \textbf{11}$\times$ less GPU hours with comparable performance to the best baseline.
The code repository for reproducing our experiments can be found at \textcolor{black}{\url{https://github.com/VirtuosoResearch/ScalableMTL}}.

\smallskip
\noindent\textbf{Summary of Contributions:}
We design an efficient algorithm, \algo, for estimating the task affinity scores of a multitask learning algorithm. The key idea of \algo{} is to trade off multitask pre-training, which is computationally expensive, with gradient-based estimation for fine-tuning, whose computation is lightweight. We then design a clustering algorithm on top of the estimation procedure for downstream multitask optimization. Through a detailed experimental study, we demonstrate that our overall algorithm, \acronym{}, significantly speeds up full model training while delivering comparable performance.

\smallskip
\noindent\textbf{Organization:}
We briefly touch on related work and then provide the technical preliminaries for the rest of the paper. In section~\ref{sec_approach}, we outline our task affinity estimation procedure \algo, along with a theoretical error analysis for the estimation error. Then, we present the clustering approach for task grouping and the overall algorithm \acronym~in Section~\ref{sec_evaluation}. Finally, we provide a thorough empirical evaluation of the \acronym~algorithm for various multitask learning settings in Section~\ref{sec_exps}.

\subsection{Related Work}

Multitask learning is a fundamental problem with many applications, such as federated learning \cite{smith2017federated}, road safety modeling \cite{nippani2024graph}, and language model fine-tuning \cite{longpre2023flan}. This problem has been studied since the early literature of data mining \cite{ben2002theoretical}.
As the number of tasks increases, modeling task relationships becomes increasingly complex and challenging \cite{ma2018modeling,zhu2021transfer}.
These relationships are influenced by data distribution characteristics, including covariate and label shifts \cite{wu2019understanding}.
Thus, designing optimization algorithms for multitask learning is challenging \cite{li2023identification,li2023boosting}.
\emph{We contribute to this literature by proposing a new approach to significantly speed up the computation of task affinity scores for modeling task relationships.}
We now proceed to discuss several lines of work that are most related to ours.

\smallskip
\noindent\textbf{Task Similarity Measures.}
Previous works \cite{standley2020tasks,fifty2021efficiently} estimate task affinities between every pair of tasks. The computation complexity of such methods scales quadratically with the number of tasks. Another approach is to use task embeddings \cite{vu2021spot}, i.e., training one model on each task and measuring the cosine similarity between the model weights. Although this approach scales linearly with the number of tasks, the measures tend to be noisy.
Intuitively, if two tasks are similar, their gradients should exhibit higher cosine similarity.
This idea can be implemented to balance training by dynamically tuning gradient magnitudes \cite{chen2018gradnorm}, or to project the gradients noto the span of other tasks' gradients that have a conflicting gradient \cite{fifty2021efficiently}.
The same idea can also be implemented to choose auxiliary tasks most beneficial for a primary task \cite{dery2021auxiliary}.
Similarity measures based on feature representations of tasks have also been applied to grouping tasks
\cite{sherif2023stg} and used to predict task transferabilities \cite{ayman2023task}.  
The main advantage of these approaches is their efficiency, as only a single multitask model needs to be trained.
The downside is that the gradients can be noisy during a stochastic training procedure.
For example, \citet{azorin2023s} empirically observed that representation and gradient similarity measures do not consistently correlate with actual MTL performance. Thus, a more accurate approach is to build measures that approximate multitask outcomes directly; see recent work on designing surrogate models for multitask learning systems \cite{li2023identification,li2023boosting}.

\smallskip
\noindent\textbf{Transferability Estimation.}
There have also been developments on information theoretic measures of transferability in recent literature.
One natural idea is to evaluate conditional entropy between target pseudo labels (assigned by a pretrained source model) and real target label \cite{bao2019information}.
Log Expected Empirical Predictor \cite{nguyen2020leep} proposes a modified procedure using soft predictions from the source model.
These methods do not utilize feature embeddings in the measure \cite{vu2020exploring}; TransRate \cite{huang2022frustratingly} introduces a surrogate measure based on mutual information that also incorporates feature embeddings. 
An improved estimation method with better robustness can be achieved by shrinkage \cite{ibrahim2022newer}. 
In the fine-tuning setting, the distance between the model search and the pretrained initialization can indicate the level of generalization capability \cite{li2021improved}.
The geometry relates to the Hessian of the loss, which has been shown to correlate with the generalization performance of fine-tuned models \cite{ju2022robust}.
\citet{ju2023generalization} extend this Hessian measure to graph neural networks, which can guide the design of optimization algorithms to regularize the Hessian of neural networks \cite{ju2023noise}. 

\smallskip
\noindent\textbf{Multitask Learning Optimization Algorithms.} 
Multitask learning can be viewed as a multiobjective optimization problem \cite{papadimitriou2000approximability}, where the goal is to identify the Pareto frontier among multiple objectives \cite{sener2018multi}.
One common MTL optimization algorithm is to reweight task losses and optimize a weighted combination of task losses \cite{liu2022auto, royer2024scalarization}.
Our goal is to maximize the averaged prediction performance of all tasks. Thus, we are interested in partitioning the tasks into similar groups, where tasks are closely related within each group and can differ significantly across groups.
Another interesting line of work is designing branching neural networks such as tree structures \cite{vandenhende2019branched,guo2020learning}, where each layer contains multiple modules to handle different tasks \cite{lu2017fully}.
Compared with branching methods, task grouping may be more suitable for handling many tasks (like hundreds to thousands). In this regime, negative interference between tasks is almost unavoidable, and clustering tasks into similar groups could provide a more efficient strategy than designing a single neural network that handles all tasks.

\smallskip
\noindent\textbf{Influence Functions.}
There is a line of work estimating the influence of adding or removing one sample on the whole dataset.
Influence functions \cite{koh2017understanding} based on efficient approximation of the Hessian inverse provide one way to approximate this.
Random sampling-based approaches to measuring leave-one-out influence have also been studied \cite{ilyas2022datamodels,park2023trak}.
The distinction between these works and us is we focus on task-level affinity, whereas this literature focuses on estimating the influence of a single data sample.

\smallskip
\noindent\textbf{Clustering Algorithms.}
Clustering is a fundamental aspect of machine learning. Besides SDP relaxations, linear programming relaxations are known for clustering objectives such as $k$-center. 
The integrality gap of linear programming and semidefinite programming relaxations can be analyzed when there is a separation structure in the underlying clusters \cite{awasthi2015relax}.
These approximation guarantees typically require the underlying similarity scores to satisfy a metric condition. By contrast, the task affinity matrix can easily violate the triangle inequality.
Recent work has also studied mixed integer programming for best subset selection \cite{bertsimas2016best}.
One novel contribution of this work is to make explicit a connection between multi-instruction fine-tuning and clustering.
In light of this connection, it would also be interesting to revisit hierarchical clustering and hypergraph clustering for task grouping.
For example, recent work by \citet{tsitsulin2023graph} investigates unsupervised graph clustering problems with graph neural networks.

\section{Preliminaries}\label{sec_prelim}

Suppose we are interested in making predictions on $n$ tasks. We are given a set of samples for training and testing each task. We aim to design a prediction algorithm to maximize the averaged testing performance over all the $n$ tasks simultaneously. We assume that the samples from all the tasks are supported on a joint product between a $p$-dimensional feature space $\cX$ and a label space $\cY$. To precisely discuss task relationships, we formally define what we mean by a \emph{multitask learning algorithm}.

\begin{definition}[Multitask learning algorithms] For any subset $S \subseteq \set{1, 2, \dots, n}$, a multitask learning algorithm $f$ takes the training data of all the tasks in $S$ and combines them in a joint training procedure. Then, the (jointly trained) model is tested on each task $t \in S$. In the end, a test result is obtained for each $t$. Let us denote the test result as $f(S, t)$. Thus, the algorithm's output will include $\abs{S}$ results for any subset $S$, one for each $t \in S$.
\end{definition}

Given a multitask learning algorithm, the transfer between the $n$ tasks can be viewed through the results of $f$, applied to combinations of tasks as subsets.
This notion of transfer underlies many existing multitask learning systems.
We give two examples below, which are used in prior works to tackle task transfer in complex visual systems \cite{zamir2018taskonomy,standley2020tasks}.

\begin{example}[Pairwise task affinity]\label{ex_pairwise}
    Consider two tasks such as $i$ and $j$.
    Given a multitask learning algorithm $f$, one can mix the training data of tasks $i, j$, using SGD to train a shared encoder and task-specific prediction heads.
    If we compute the pairwise task affinity for all pairs of tasks $1 \le i \le j \le n$, then we get an $n$ by $n$ task affinity matrix $T$, where $T_{i, j} = f(\set{i, j}, i)$.
\end{example}

\begin{example}[High-order task affinity]\label{ex_ho}
    Next, we discuss higher-order task affinity, analogous to sampling features in random forests.
    First, fix an integer $m$, which is the number of subsets we would like to sample (e.g., analogous to the number of decision trees in a random forest).
    We independently sample $m$ subsets out of the set $\set{1, 2, \ldots, n}$, each subset having a size of $\alpha$, chosen uniformly over all such subsets.
    Let us denoted the $m$ subsets as $S_1, S_2, \ldots, S_m$.
    Then, compute $f(S_k, j)$, for every $k = 1, 2, \dots, m$, and $j = 1, \dots, \alpha$. 
    Lastly, compute $T_{i, j}$ as the average value of $f$ among all subsets including tasks $i,j$:
    \begin{align}
        T_{i, j} = \frac{1}{n_{i,j}} \sum_{ 1\leq k \leq m:\ i \in S_k, j \in S_k} f\big(S_k, i\big), \text{ for all } 1 \leq i, j \leq n, \label{eq_higher_order_affinity}
    \end{align}
    where $n_{i,j}$ is the number of subsets that include both $i, j$. 
    This leads to another task affinity matrix $T$, which better captures the higher-order relationship among tasks.
\end{example}

In both examples, computing the task affinity matrix requires fitting $\Omega(n)$ models, given $n$ tasks.
In Example \ref{ex_pairwise}, one needs to train $\binom{n} 2$ models, one for every pair of tasks.
Then, in Example \ref{ex_ho}, a total of $m = \Omega(n \log n)$ models are required, each for a subset of tasks.
This raises the question of whether one can approximate the results of a multitask learning algorithm by designing a more efficient computational method.

Specifically, given a multitask learning algorithm $f$ and a collection of subsets $S_1, S_2, \dots, S_m  \subseteq \set{1, \ldots, n}$, can we quickly estimate the task affinity corresponding to $f(S_i, j)$, for any $i = 1, 2, \dots, m$ and any $j \in S_i$ quickly (e.g. without fully training a model for each subset)?
Do these task affinity estimates accurately approximate the affinity one would get from fully trained models? Moreover, are the estimates useful in the downstream task grouping setup?

\section{Task Affinity Estimation}\label{sec_approach}

We now describe a new method for estimating task affinity scores.
To circumvent the cost of full-model training, we describe an empirical observation regarding pre-training and fine-tuning.
Then, we present our approach to estimating fine-tuned model parameters for task subsets.
Additionally, we use random projection to reduce the dimension of the gradients.
We provide an error analysis to justify the design of our algorithm.

\subsection{Linearization of Fine-tuned Models}

Our method is motivated by the fact that once we pre-train all the $n$ tasks to obtain a meta-initialization, this initialization can provide representations quickly adapted to the remaining tasks.
This is based on the premise that the underlying tasks share structural similarities in multitask learning. 
As the model fine-tuned to a subset of tasks stays in the affinity of the initiation, the fine-tuning procedure behaves like linear models locally.

To illustrate this observation, we consider three scenarios involving graph neural networks (GNNs) and transformers (BERT and T5).
We test GNNs on a multi-label prediction dataset on a YouTube graph \cite{yang2012defining}, using a 3-layer SIGN network \cite{frasca2020sign}.
This dataset includes $n = 100$ subtasks, one corresponding to the node labels of a subgraph of the whole graph.
For transformers, we take a pretrained BERT model and fine-tune it on a sentence classification dataset \cite{yu2020fine}, which contains $n = 26$ tasks.
We also use a pretrained T5-Base model and fine-tune it on a sentence classification dataset with $100$ instructions \cite{bach2022promptsource}, which has $n=100$ tasks. In each experiment, we first train a meta-initialization $\theta^{\star}$ by training on all tasks combined. Then, we fine-tune $\theta^{\star}$ on a random subset of the tasks.

We perform Taylor's expansion with $\theta^{\star}$ as the anchor point.
Let $W$ denote the fine-tuned weight.
Denote the model with $W$ and $\theta^{\star}$ as $f_W$ and $f_{\theta^{\star}}$, respectively.
For an input $x$ with label $y$, denote the output of the fine-tuned model as $f_W(x, y)$.
If $W$ is close to $\theta^{\star}$, $f_W(x, y)$ can be approximated by 
\begin{align}\label{eq_taylor}
    f_W(x, y) \approx f_{{\theta}^{\star}}(x, y) + \big[\nabla_W f_{\theta^{\star}}(x, y)\big]^{\top} (W - \theta^{\star}) + \epsilon.
\end{align}
We measure the error term $\epsilon$ and report the Residual Sum of Squares (RSS) in Table \ref{table_compare_approximation_error}:
\[ \frac{  \bignorm{f_W(x, y) - f_{\theta^{\star}}(x, y) - \nabla_W f_{\theta^{\star}}(x, y)^{\top}(W - \theta^{\star})}^2 }{\bignorm{f_{W}(x, y)}^2 }. \]
In particular, we fine-tune the meta-initialization to a subset of tasks to get weight $W$.
Then, we measure the fine-tuned distance as $\frac{\bignorm{W - \theta^{\star}}} {\bignorm{\theta^{\star}}}$.
Interestingly, our results show that the gradient-based approximation yields within 3.5\% RSS, even when the fine-tuned distance is up to 10\%.
In particular, viewing $W$ as the decision variables, Eq. \eqref{eq_taylor} is a linear model with the gradient $\nabla_W f_{\theta^{\star}}(x, y)$ as the feature vector.

\begin{table}[t!]
\centering
\caption{Measuring Taylor's expansion error for models fine-tuned from an initialization pre-trained on all tasks. The results are averaged over $100$ random task subsets.}\label{table_compare_approximation_error}
\vspace{-0.05in}
{\small\begin{tabular}{@{}cccccc@{}}
\toprule
\multicolumn{2}{c}{GNN} &  \multicolumn{2}{c}{BERT} & \multicolumn{2}{c}{T5} \\ \cmidrule(l){1-2} \cmidrule(l){3-4} \cmidrule(l){5-6} 
    Distance & RSS  & Distance & RSS & Distance & RSS \\ \midrule
    1\% &  $4.2\times10^{-4}$ & 1\% &  $3.6\times10^{-6}$  & 1\% & $3.8\times10^{-6}$\\
    2\% &  $9.5\times10^{-4}$ & 2\% & $5.4\times10^{-6}$   & 2\% & $6.0\times10^{-5}$\\
    3\% &  $1.1\times10^{-3}$ & 3\% &  $3.0\times10^{-5}$  & 3\% & $3.2\times10^{-5}$\\
    4\% &  $2.5\times10^{-3}$ & 4\% & $1.5\times10^{-4}$   & 4\% & $2.6\times10^{-4}$\\
    5\% &  $6.8\times10^{-3}$ & 5\%  & $2.2\times10^{-4}$  & 5\%  & $6.3\times10^{-4}$\\
    6\% &  $7.5\times10^{-3}$ &  6\% &  $5.7\times10^{-4}$ &  6\% & $8.4\times10^{-4}$\\
    7\% &  $9.0\times10^{-3}$ &  7\%  & $9.9\times10^{-4}$ &  7\% & $1.4\times10^{-3}$\\
    8\% &  $9.3\times10^{-3}$ & 8\% &  $9.0\times10^{-4}$  & 8\%  & $2.5\times10^{-3}$\\
    9\% &  $1.2\times10^{-2}$   & 9\%  &  $2.2\times10^{-3}$  & 9\% & $3.3\times10^{-3}$\\
    10\% & $3.4\times10^{-2}$ &  10\%     & $5.1\times10^{-3}$&  10\%  & $4.1\times10^{-3}$\\
\bottomrule
\end{tabular}}
\end{table}

\begin{remark}[Second-order approximation]
It is natural to ask if a second-order approximation can further reduce Taylor's expansion error. Notice that there is a tradeoff between approximation quality and computation cost. 
Based on our preliminary test of the Hessian approximation, it can indeed reduce estimation error; however, this requires computing Hessian-gradient products.
The premise is that the underlying tasks share a structural similarity, like in community detection, where clusters have higher densities. Our experiments found that 94\% of models fine-tuned for random task subsets remain <10\% distance to initialization (on the Youtube and RTE datasets), suggesting that the first-order approximation is generally sufficient.
\end{remark}

\subsection{Gradient-based Estimation}

We now describe our algorithm, which builds on the above linearization property, using logistic regression with gradients as features.
It also includes dimension reduction, as described below.

\smallskip
\textbf{(1) Estimating fine-tuned model parameters:}
In the following discussion, we focus on binary classification, such that $y_i \in \set{+1, -1}$.
See Remark \ref{rem_ext} for extensions to multiple classification and regression.
Recall the gradient-based approximation of $f_W(x_i, y_i)$, given the input $(x_i, y_i)$: %
\[ \nabla_W f_{\theta^{\star}}(g_i, y_i)^{\top} (W - \theta^{\star}) + f_{\theta^{\star}}(x_i, y_i) \]
Let us denote $\nabla_{W} f_{\theta^{\star}}(x_i, y_i)$ as $g_i$ and $- y_i f_{\theta^{\star}}(x_i, y_i)$ as $b_i$, for any $i$.
Using logistic loss, we can write down the loss function as
\begin{align}\label{eq_loss_logr}
    \tilde \ell_W(g_i, y_i) = \log\left( 1 + \exp\left(-y_i  g_i^{\top} (W - \theta^{\star}) + b_i \right) \right),
\end{align}
for $W \in\real^p$.
Denote the combined data set in the task subset $S$ as
\[ \cD_S = \set{(x_1, y_1), \ldots, (x_{n_S}, y_{n_S})}, \]
where $n_S$ is the combined number of data samples in the set $\cD_S$.

The main idea is to solve a logistic regression problem with $g_i$ being the feature vector and $y_i$ being the response label.
However, recall that the dimension of $g_i$ is the same as the number of parameters in a neural network, which could be tens of millions.
Thus, we introduce a dimension reduction procedure that does not lose much precision.

\textbf{(2) Dimension reduction:} %
We use the Johnson-Lindenstrauss random projection \cite{johnson1984extensions}, which projects the gradients to a much lower dimension before solving the logistic regression.
Let $P$ be a $p$ by $d$ Gaussian random matrix, whose entries are independently sampled from a Gaussian $N(0, d^{-1})$.
We project the gradient from dimension $p$ onto dimension $d$ as
$\tilde g_i = P^{\top} g_i$.
Then, we solve the following logistic regression, which is now in dimension $d$: 
\begin{align}\label{eq_logr}
    \hat W_d \leftarrow \mathop{\arg \min}_{W\in\real^d} \hat L(W) = \frac 1 {n_{S}} \sum_{i = 1}^{n_{S}} \tilde\ell_W(\tilde g_i, y_i).
\end{align}
Lastly, we set $\hat W_S$ as $P \hat W_d + \theta^{\star}$ to map the projected solution to the $p$-dimensional space.
$\hat W_S$ is the estimated model parameter for fine-tuning $\theta^{\star}$ with task subset $S$.

\textbf{(3) Averaging over an ensemble:} To reduce the above estimation's variance, we also add a model averaging step.
In particular, we train several meta-initializations and repeat the above estimation procedure.
We average the estimated scores within the ensemble.

\begin{algorithm}[t!]
\raggedright
\caption{\algo~(Gradient-based Task Affinity Estimation)}\label{alg_estimate_mtl_performance}
\textbf{Input}: A list of subsets $S_1, S_2, \dots, S_m \subseteq \set{1, 2, \dots, n}$, and their training and testing datasets\\
\textbf{Require:} Initializations ${\theta^{\star}_1, \theta^{\star}_2, \ldots, \theta^{\star}_M}$; projected dimension $d$\\
\textbf{Output}: Estimated scores $\hat f(S_i, j)$ for every $i = 1, 2, \dots, m$, $j \in S_i$
\begin{algorithmic}[1]
    \For{$k = 1, \dots, M$}
        \State Let $P$ be a $p$ by $d$ Gaussian random matrix $\sim N(0, d^{-1})$
        \State Project the gradient of every training example $(x, y)$ as \[ \tilde{g} = P^{\top} \nabla_W f_{\theta^{\star}_k}(x, y) \]
        \For{$i = 1, \dots, m$}
            \State Run logistic regression with $\set{\tilde{g}, y}$ on all the samples belong to tasks in $S_i$ to obtain $\hat W_{d}$. Let \begin{align} \label{eq_hatS} \hat W_{S_i} = \theta_k^{\star} + P \hat W_d \end{align}
            \State Evaluate $f^{(k)}_{\hat W_{S_i}}(S_i, j)$, for every $j \in S_i$
        \EndFor
    \EndFor
    \State Average over the ensemble as \[ \hat f(S_i, j) = \frac{1}{M} \sum_{k=1}^M  f^{(k)}_{\hat W_{S_i}}(S_i, j), \text{ for every } j \in S_i \]
\end{algorithmic}
\end{algorithm}

\smallskip
We summarize the entire procedure in Algorithm \ref{alg_estimate_mtl_performance} with all three steps.
Let us compare the running time complexity between this estimation and one that uses full training to get $f(S_i, j)$ instead:
\begin{itemize}[leftmargin=*]
    \item In our estimation, we need $M$ full training, plus $O(n)$ gradient evaluations and solving logistic regression $m$ times.
    \item If we were to compute $f$, we need $m$ full model training instead.
\end{itemize}
Typically, $M = O(1)$, while $m = \Omega(n)$ or even $O(n^2)$ in downstream use cases.
Thus, our estimation algorithm reduces $\Omega(n)$ full-model training to only $O(1)$.
The tradeoff is that we require $O(n)$ gradient evaluations (to retrieve the gradients on all tasks) plus solving logistic regression $m$ times.
As shown below, the random projection helps reduce the dimension of the logistic regression problem to $O(\log p)$ dimension, which is much cheaper.
This is in terms of the asymptotic complexity.
In Section \ref{sec_evaluate_task_affinity}, we materialize the constants to compare the number of FLOPs during training.

\begin{remark}[Extension to multiple classification or regression]\label{rem_ext}
    The above procedure can be extended to deal with multiple classifications.
    This requires setting up one prediction vector for each class; The rest remains the same.
    The procedure also applies to regression by using mean squared error instead.
\end{remark}

\subsection{Error Bounds}\label{sec_bound}

We now show that the error introduced by approximations in \algo{} is bounded. Specifically, we use the Johnson-Lindenstrauss Lemma to argue that as $d$ increases, the random projection yields a minimizer whose quality is not much worse than the solution without the projection.
We will assume that the averaged Taylor's expansion error is at most $\delta$ across the entire data set of every task.
Additionally, we assume the search procedure occurs within a bounded space of radius $D$.
Lastly, in the pretrained initialization, each gradient vector's Euclidean norm is at most $G$.
With these conditions, we state the error bounds for \algo~as follows.

\begin{proposition}\label{prop_jl}
    Let $\cD$ be a search space whose radius is at most $D$.
    Suppose the gradient of $f_{\theta^{\star}}$ at the initialization $\theta^{\star}$ in the training set is at most $G$ in Euclidean norm.
    For each task $i = 1, 2, \dots, n$, let $T_i$ denote the training data. Suppose that for every $i$,
    \begin{align*}
        {\frac 1 {\abs{T_i}} \sum_{(x, y) \in T_i} \abs{f_W(x, y) - f_{\theta^{\star}}(x, y) - \nabla_W f_{\theta^{\star}}(x, y)^{\top} (W - \theta^{\star}) }} \le \delta.
    \end{align*}
    Provided that $d = O\Big(\frac{\log p}{\epsilon^2}\Big)$, the training loss of ~$\hat W_S$ is bounded away from the minimum training loss for any $S \subseteq \set{1,2,\dots,n}$ as
    \begin{align}\label{eq_jl_gua}
        \hat L(\hat W_S) \le \min_{W\in\cD} \hat L(W) + 2\delta +  {4 G D } {\epsilon}.
    \end{align}
\end{proposition}

The proof, given in Appendix \ref{proof_prop}, uses the Johnson-Lindenstrauss Lemma \cite{johnson1984extensions}.
In particular, since the logistic loss is $1$-Lipschitz continuous, we can relate $\hat L(\hat W_S)$ to $\min \hat L(W)$.
The errors introduced by random projection and Taylor's expansion can be bounded using the JL Lemma and the bound on Taylor's expansion error, respectively.
Further, our experiments in Table \ref{table_compare_approximation_error} suggest that $\delta$ is relatively small in practice.
Thus, as $\epsilon$ goes to zero, Eq. \eqref{eq_jl_gua} guarantees the gap between $\hat L(\hat W_S)$ and $\min \hat L(W)$ will be small.

\section{Task Affinity based Grouping}\label{sec_evaluation}

We now describe a clustering algorithm to partition the $n$ tasks into $k$ disjoint subsets.
Given an $n$ by $n$ task affinity matrix $T$, we will find a clustering that maximizes the average density of all clusters.
Concretely, let $C_1, \dots, C_k$ be a disjoint partition of $[n]$. 
Let $v_1, \dots, v_k$ be a $0$-$1$ vector indicating whether a task is in one cluster.
The average density of this clustering can be written as:
\begin{align}
   \frac 1 k \sum_{i=1}^{k} \frac{v_i^{\top} T v_i}{v_i^{\top} v_i}. \label{eq_obj} %
\end{align}
This integral objective is NP-hard to optimize in general (in particular, geometric clustering is a special case \cite{aloise2009np}).

We design a Semi-Definite Programming (SDP) relaxation and then round the SDP solution to a clustering.
Let us denote the assignment variables as an $n \times k$ matrix $V$, such that each entry $V_{i,j}$ indicates whether a task $i$ belongs to a cluster $j$, for every $i = 1, \dots, n$, $j = 1, \dots, k$.
Moreover, let the $j$th column of $V$, the characteristic vector of the $j$-th cluster, be denoted as $v_j$.
Under this assignment, the sum of $V_{i, j}$ across any task $i$ must be one, as we allow one task to be assigned in a single group. By contrast, the sum of $V_{i, j}$ across $C_j$ is the number of tasks assigned to  $C_j$, which is at least one.

Let $e$ denote the all-ones vector. We state an integer program to maximize the average density of all $k$ clusters as follows
\begin{align}
    \max_{V\in\real^{n\times k}}~~&  \BigInner{T}{ \frac 1 k\sum_{j=1}^k \frac{v_j v_j^{\top}}{v_j^{\top}v_j}} \nonumber \\
    & Ve = e, \sum_{i=1}^n V_{i,j} \geq 1 \text{ for } 1 \leq j \leq k \nonumber \\
    & V_{i, j} \in \set{0, 1}, \text { for any } 1 \le i\le n, 1 \le j\le k. \label{eq_clustering_objective}
\end{align}
Note that $v_iv_i^{\top}$ is a rank-one semidefinite matrix.
Let us denote the sum of them (normalized by $v_i^{\top}v_i$) as the following new variable
\begin{align}
    X = \sum_{j=1}^k \frac{v_j v_j^{\top}}{v_j^{\top}v_j}. \label{eq_define_x}
\end{align}
$X$ has rank $k$ since it is the sum of $k$ rank-$1$ matrices, and the $v_i$'s are orthogonal to each other. Additionally, its trace is equal to $k$ because the trace of $\frac{v_j v_j^{\top}}{v_j^{\top}v_j}$ is one for any $j$.
Second, one can verify that the entries of every row of ${X}$ sum up to one.
Removing the $0$-$1$ integer constraint, we derive a rank-constrained problem as
\begin{align*}
    \max_{X\in\real^{n\times n}}~~& \inner{T}{X} \\
        & X e = e, \tr[X] = k, \textup{rank}(X) = k \\
        & X \geq 0,  X \succeq 0.
\end{align*}
Further relaxing the rank constraint (while keeping the trace constraint) leads to a convex program, which can be solved efficiently.

Given a solution of the SDP, denoted as $\hat X$, the last step is to round $\hat X$ into an integer solution.
We set a threshold $\lambda$ such that if $\hat X_{u,v} \ge \lambda$, tasks $u$ and $v$ are assigned to the same cluster.
In the experiments, we set $\lambda$ as $c/n$ for a constant $c\ge 1$, since $\hat X_{u,v}$ should be $\frac{1}{|C_i|}$ when they are in the same cluster with $|C_i|< n$.
Thus, the intra-cluster distance must always be at least $\lambda$ with the assignment.

We provide the entire procedure in Algorithm \ref{alg_approximate_grouping}, which uses Algorithm \ref{alg_estimate_mtl_performance} as a subroutine to estimate the task affinity scores. %

\begin{algorithm}[t!]
    \caption{\acronym~(Gradient-based Task Affinity Grouping)}\label{alg_approximate_grouping}
    \raggedright
    \textbf{Input}: $n$ tasks along with their training and testing datasets; number of desired clusters $k$\\
    \textbf{Require:} Number of subsets $m$ and size $\alpha$, rounding threshold $\lambda$, number of trials $M$, projected dimension $d$\\
    \textbf{Output}: A disjoint partition of $[n]$ as $\cC$\\
    \begin{algorithmic}[1]
        \State Run $f([n], \cdot)$ for $M$ times independently to obtain $\theta^{\star}_1, \dots, \theta^{\star}_M$
        \State Sample $m$ subsets of size $\alpha$ from $[n]$
        \State $\set{\hat f(S_i, j): 1\le i \le m, j \in S_i} \leftarrow$ \algo($\theta^{\star}_1, \dots, \theta^{\star}_M; d$)
        \State Construct an $n$ by $n$ affinity matrix $T$ following equation \eqref{eq_higher_order_affinity}
        \State Obtain $\hat X$ by solving problem
        \begin{align}
            \max_{X \in \real^{n\times n}}~~& \inner{T}{X} \label{eq_sdp_clustering} \\
                 & X e = e,  \tr[X] = k \nonumber \\
                 & X \geq 0, X \succeq 0.\nonumber
        \end{align}
        \State Round the solution $\hat X$ into clusters using the threshold ${\lambda}$
    \end{algorithmic}
\end{algorithm}

\begin{example}[Discussion about alternative clustering algorithms]\label{ex_planted} 
    A natural question is using alternative algorithms such as spectral clustering or Lloyd's clustering.
    We find that these algorithms are not as robust as the SDP relaxation because the scale of the loss values varies across rows for different tasks.
    We describe a toy example to illustrate. Suppose $T$ is a $6$ by $6$ matrix involving three clusters $C_1, C_2, C_3$ of size $2$ each.
    The affinity in $C_1$ is $7$, while the affinity scores in $C_2$ and $C_3$ are $20, 19$, respectively.
    We find that both spectral clustering and Lloyd's clustering will group $C_2$ and $C_3$ together, while the SDP relaxation manages to separate them.
    See Figure \ref{fig_clustering_synthetic} for an illustration.
    For this reason, we use the SDP relaxation in \acronym.
\end{example}%
\begin{figure}[!ht]
    \centering
    \begin{minipage}[b]{0.245\textwidth}
        \centering
        \includegraphics[width=0.65\textwidth]{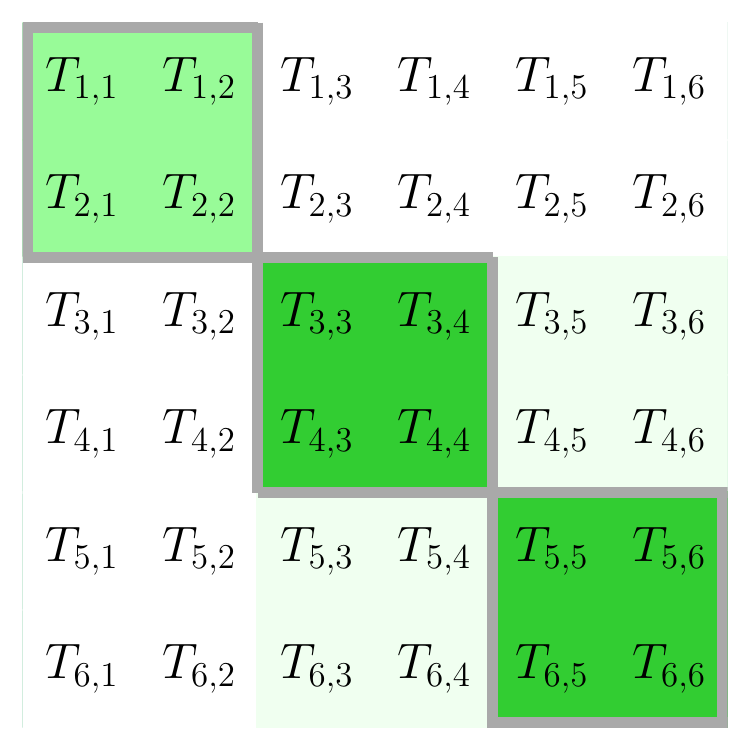}
        \vspace{-0.00in}
        \subcaption{SDP relaxation}\label{fig_sdp_synthetic}
    \end{minipage}%
    \begin{minipage}[b]{0.245\textwidth}
        \centering
        \includegraphics[width=0.65\textwidth]{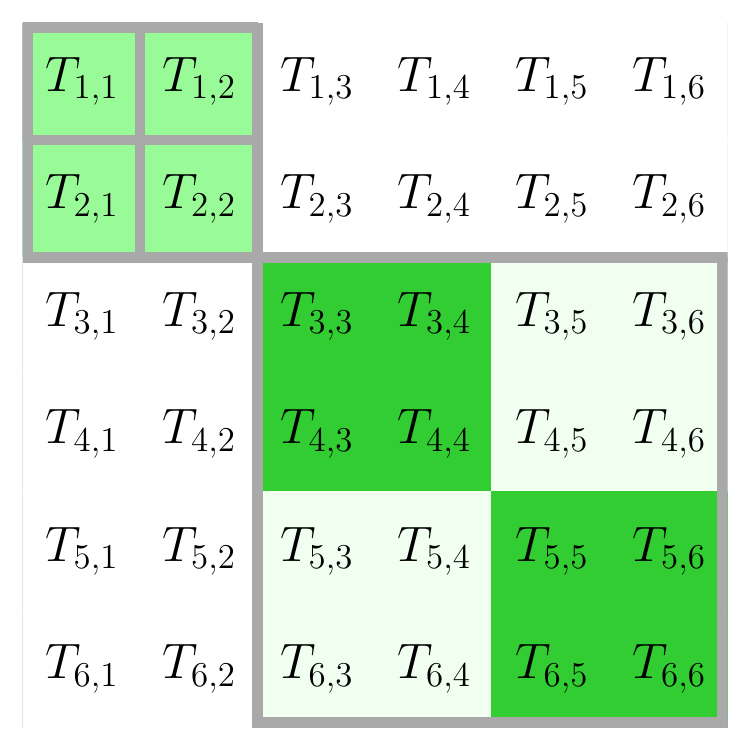}
        \vspace{-0.00in}       
        \subcaption{Spectral/Lloyd's clustering}\label{fig_sl}
    \end{minipage}
    \caption{We compare the SDP relaxation with spectral and Lloyd's clustering in a toy example.
    There are three clusters, with the second and third clusters having higher densities than the first.
    The black solid line illustrates the clusters yielded by each algorithm. 
    As shown in Fig. \ref{fig_sl}, spectral and Lloyd's clustering group the high-affinity clusters together.
    Fig. \ref{fig_sdp_synthetic} shows the SDP relaxation separates them correctly.}\label{fig_clustering_synthetic}
\end{figure}
\begin{remark}[Approximation ratio of the SDP relaxation]
    A natural question is whether one can quantify the approximation ratio of the SDP relaxation \eqref{eq_sdp_clustering}.
    Although this is a well-studied problem in approximation algorithms \cite{ailon2008aggregating}, task affinity violates the metric condition typically required to obtain guarantees in this literature.
    In particular, the triangle inequality $T_{i, j} + T_{j, k} \ge T_{i, k}$ is violated.
    It is possible that by assuming intra-cluster separation (see, e.g., \citet{awasthi2015relax}), one might be able to analyze the SDP theoretically. This is left for future work.
\end{remark}
\smallskip
\begin{remark}[Further variants of \acronym]
    While we focus on the task grouping problem, the idea can be used to speed up forward and backward selection.
    We set the list of subsets in Algorithm \ref{alg_estimate_mtl_performance} as $\set{1}, \set{2}, \dots, \set{n}$.
    Suppose we select task $3$. Then, in the next round, we set the list of subsets as $\set{3, 1}, \set{3, 2,}, \dots, \set{3, n}$.
    And so on.
\end{remark}

\section{Experiments}\label{sec_exps}

We now validate \algo{} and \acronym{} across various settings.
The evaluation focuses on the following key questions. 
Does the estimation procedure accurately approximate the target task affinity scores? How does the running time compare to the full computation required to obtain these scores?
Third, do the estimated affinity scores combined with the clustering algorithm work well in downstream use cases? 

Our experiments show that \algo{} approximates the true task affinities (based on full model training) within a relative error of less than 2.7\%, using less than 3\% of the computational cost of full training. Further, \acronym{} achieves comparable downstream accuracy to existing methods in two canonical applications, multi-label classification on graphs and language model fine-tuning, while using 32.8$\times$ fewer FLOPs. Lastly, we discuss the parameters and the steps as part of our algorithm, including the comparison with alternative clustering.

\subsection{Experimental Setup}

\subsubsection{Evaluation settings.} We note that our algorithm applies to a wide range of multitask learning scenarios. For a representative evaluation, we focus on multi-label prediction on graphs and language model fine-tuning. In the first setting, each labeling task corresponds to a subgraph within a graph. Given a seed set of each labeling as the training set, the goal is to identify the remaining nodes of the subgraph. This can be cast as multitask learning by viewing each labeling as a binary classification task. The objective is to optimize the average accuracy of all the labeling tasks.

The second setting involves fine-tuning language models using human-designed instructions, known as instruction fine-tuning. Each instruction corresponds to a prompt. Typically, a data set can come up with many relevant instructions, some of which are more relevant to a subset of tasks than others \cite{longpre2023flan}. Thus, a natural question is to select the instructions that are more relevant to the downstream task, which can be formulated using multitask learning. In particular, we view each instruction tuning as a single task.
While we focus on these two applications, it is conceivable that our algorithm can be used in other related applications.

\subsubsection{Datasets and models.}
We use social network datasets with community labels for multi-label prediction on graphs. We select four graphs from SNAP \cite{yang2012defining} (Amazon, YouTube, DBLP, and LiveJournal), while we expect similar results to hold on other graphs. The number of nodes in these four graphs ranges from 3k to 57k; the number of edges ranges from 20k to 1M. 
For each graph, we pick 100 (largest) communities corresponding to $n = 100$ tasks.
For preprocessing, we randomly sample 10\% of nodes from each community subgraph as positive training samples and 10\% of nodes outside the subgraph as negative samples. From the remaining data, 20\% is randomly sampled for validation. We evaluate performance using the macro $F_1$-score on the test set \cite{yang2013overlapping}.

Next, we examine the running time scaling of our algorithm on a large graph (the Orkut network), which has 395k nodes, 21M edges, and a total of 500 communities.
We use a 3-layer SIGN model \citep{frasca2020sign} with a fixed width of 256 as the encoder in the MTL models, which is more efficient to train than GCN.

For fine-tuning language models, we use two text classification datasets from SuperGLUE \cite{wang2019superglue}, specifically RTE and WiC. Each dataset includes 100 instructions, with ten sourced from \citet{bach2022promptsource} and 90 generated using the automatic instruction generation method in \cite{zhou2022large}. Thus, each dataset has 100 tasks, each corresponding to fine-tuning with one instruction.
We use T5-Base \cite{roberts2019exploring} as the encoder for the MTL model.
The choice of this encoder is without loss of generality, as we expect similar results to hold on other encoders.

Put together, our experiment covers seven different datasets in total, spanning medium- and large-scale instances, with the largest dataset containing $500$ tasks.

\subsubsection{Evaluation metrics.} We assess the accuracy of estimated task affinity by measuring the distance between our estimated task affinities and those computed from fully trained models.

For task grouping, we evaluate the accuracy averaged over all tasks when training a collection of networks, each on a subset of tasks.  The accuracy metric is task-dependent, such as zero-one accuracy or the $F_1$-score, depending on the setting.

Lastly, we measure each method's total number of FLoating-point OPerations, namely FLOPs. In addition, we report the number of GPU hours evaluated on a single Nvidia RTX6000 GPU.

\subsection{Task Affinity Estimation} \label{sec_evaluate_task_affinity}

We now report the results of our estimation procedure. We regard the task affinity scores computed from fully trained models as the target, denoted as ${T}^{\star}$.
Then, after running \algo, we compute the affinity matrix $T$, and measure the relative distance between $T$ and ${T}^{\star}$ as:
\[ \textup{Distance}(T, T^{\star}) = \frac {\bignormFro{T - {T}^{\star}}^2} {\bignormFro{{T}^{\star}}^2}. \]
We evaluate the relative distance on the YouTube graph, which contains $n = 100$ labeling tasks.

As for the computation cost, our procedure has three parts: (i) training $M$ meta-initializations, each on the combination of all tasks; (ii) For each meta-initialization, computing the gradients on all training examples and projecting the gradients to a lower-dimension; (iii) Solving logistic regression on projected gradients for a subset of task and evaluate the performance on each task in the subset. We report the computation in terms of FLOPs using our algorithm to compute $T$ and fully training models to compute $T^{\star}$.

\subsubsection{Accelerating pairwise task affinity computation.}

First, we train a separate multitask model on each pair of tasks to compute $T^{\star}$.
We report the distance metric and the number of FLOPs between fully-trained models (to compute $T^{\star}$) and our algorithm in Table \ref{table_compare_distance}.

To explain our findings, we set the number of meta-initializations to $M = 1$ and vary the projection dimension $d$ among 50, 100, 200, and 400.
We note that all these values yield an estimation of $T^{\star}$ within 11\% distance. As expected, increasing $d$ leads to better estimation.
After $d$ increases above $200$, the distance metric stabilizes to around \textbf{5.7\%}. Thus, we set $d$ as 200 in the remaining experiments.
As a remark, this is approximately $15\log(p)$, where $p = 683,370$ in this experiment, aligning with our analysis in Proposition \ref{prop_jl}.
Remarkably, under this setting, \algo{} uses \textbf{3.5} GPU hours and achieves \textbf{130}$\times$ less computation compared to fully-trained models!

Next, we fix $d = 200$ while increasing $M$ up to $9$. This further reduces the distance metric to \textbf{5.4}\%, with \textbf{45.0}$\times$ less compute cost. We observe diminishing returns from the ensemble once $M$ goes beyond $5$. Thus, we will set $M$ as $5$ in the remaining experiments. This uses 17.6 GPU hours and 44.9$\times$ less computation than fully-trained models.

\begin{table}[t!]
\centering
\caption{We report the distance between our estimated task affinity and $T^{\star}$, computed on the YouTube graph. For interpreting the computation cost, we report the ratio between the number of FLOPs to compute $T^{\star}$ divided by the number of FLOPs of our algorithm. Recall from Algorithm \ref{alg_estimate_mtl_performance} that $M$ is the number of meta-initializations, and $d$ is the random projection dimension. }\label{table_compare_distance} %
\vspace{-0.1in}
{\small\begin{tabular}{@{}cccccccc@{}}
\toprule
  & & \multicolumn{2}{c}{Pairwise task affinity} &  \multicolumn{2}{c}{Higher-order task affinity} \\ \cmidrule(l){3-4} \cmidrule(l){5-6}
$d$ & $M$ & Distance & Speedup  & Distance & Speedup \\ \midrule
50 & 1 & 10.8\% & 132.1$\times$  & 5.5\% & 72.4$\times$\\
100 & 1 & 10.2\% & 131.6$\times$ & 5.0\% & 72.2$\times$ \\
200 & 1 & 7.0\% & 130.4$\times$  &  3.5\% & 71.4$\times$\\
400 & 1 & 6.8\% & 128.2$\times$  &  3.4\% & 69.8$\times$\\
200 & 3 & 6.1\% & 66.9$\times$ &  2.7\% & 45.0$\times$ \\
\textbf{200} & \textbf{5} & \textbf{5.7}\% & \textbf{45.0}$\times$ &  \textbf{2.7}\% & \textbf{32.8}$\times$ \\
200 & 7 & 5.4\% & 33.9$\times$ &  2.6\% & 25.9$\times$ \\
200 & 9 & 5.4\% & 27.2$\times$ &  2.4\% & 21.3$\times$ \\
\bottomrule
\end{tabular}}
\end{table}

\begin{figure}[t!]
    \begin{subfigure}[b]{0.22\textwidth}
        \centering
        \includegraphics[width=0.9\textwidth]{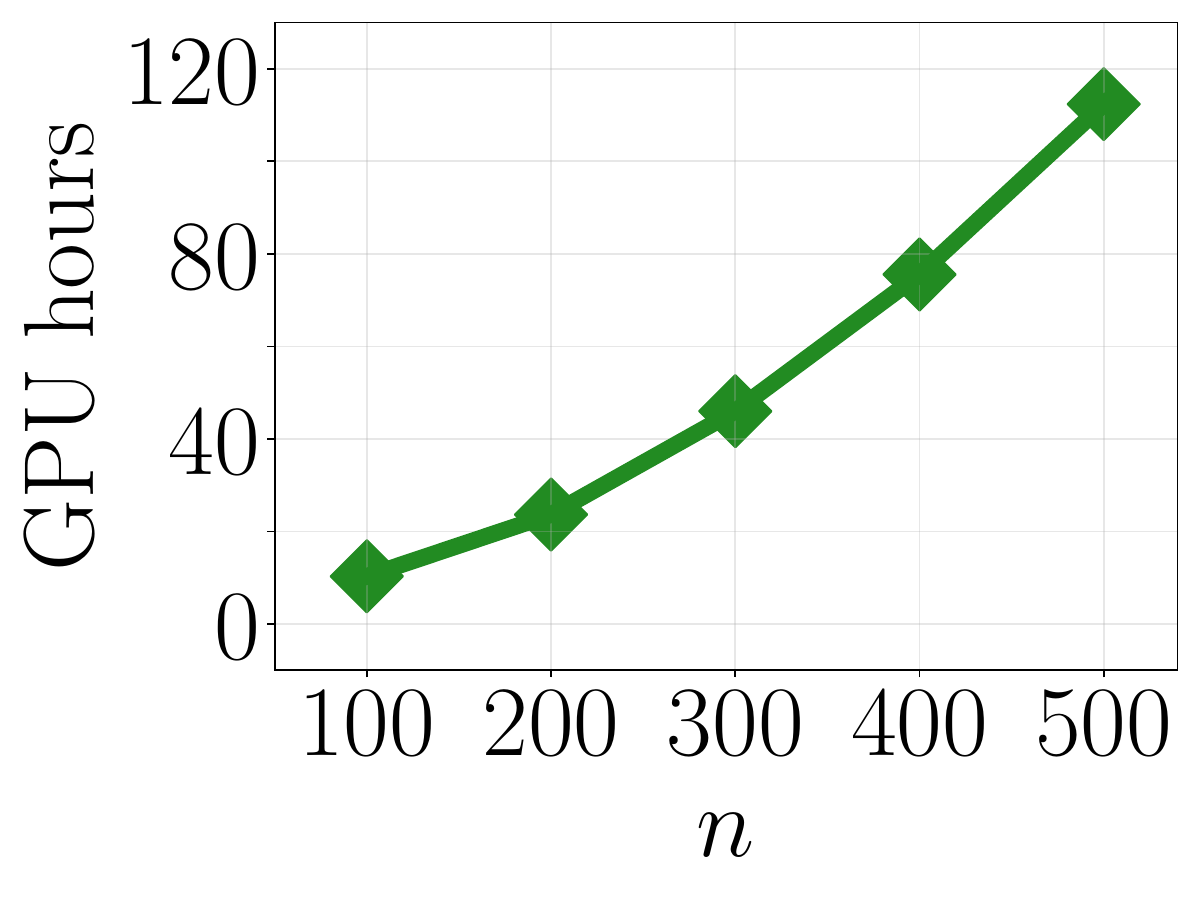}
        \vspace{-0.1in}
        \subcaption{Our estimation}
    \end{subfigure}
    \begin{subfigure}[b]{0.22\textwidth}
        \centering
        \includegraphics[width=0.9\textwidth]{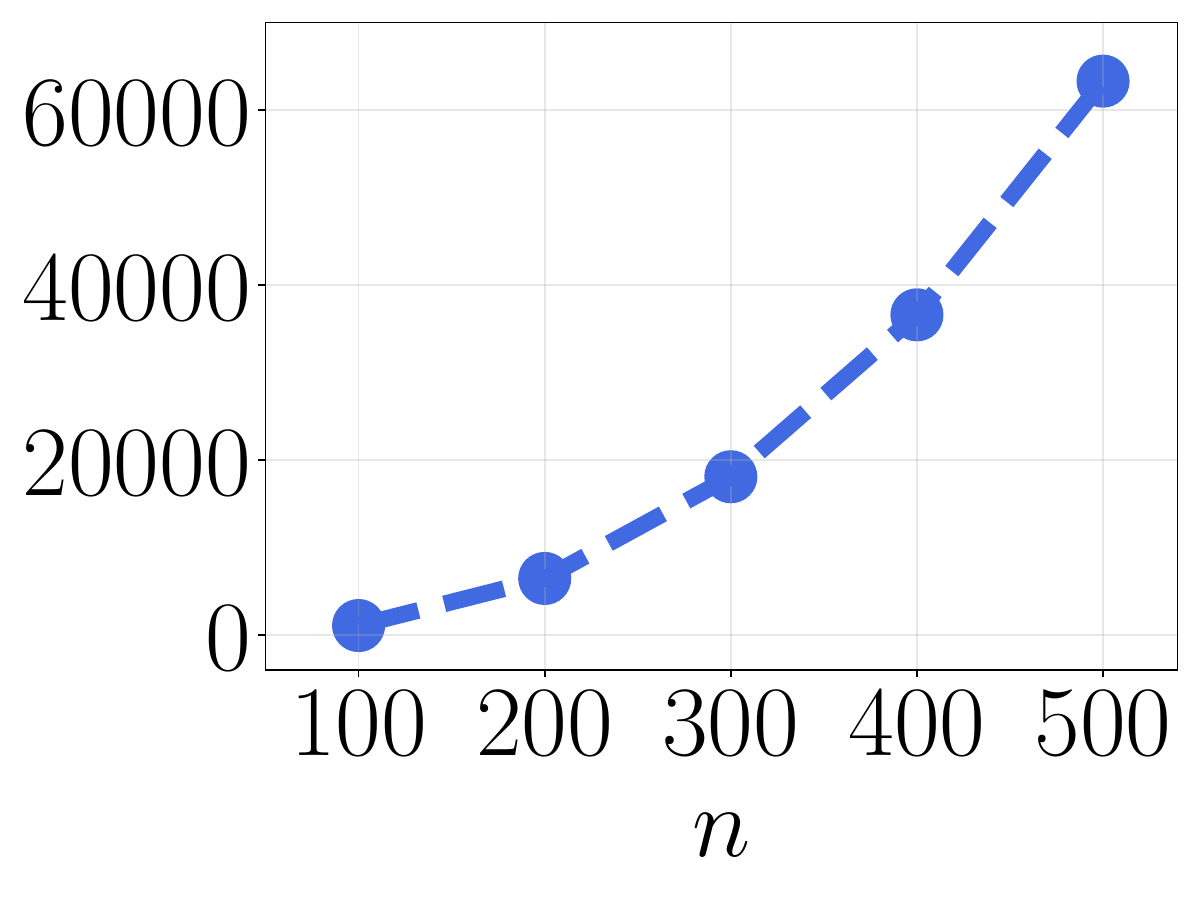}
        \vspace{-0.1in}
        \subcaption{Full training cost}
    \end{subfigure}
    \caption{The number of GPU hours vs. the number of tasks to compute pairwise affinity, evaluated on the Orkut graph up to 500 tasks. We estimate the full training cost by training on randomly sampled 2000 subsets of tasks. 
    }\label{fig_scaling}
\end{figure}

\subsubsection{Accelerating higher-order task affinity computation.}
We note qualitatively similar results for approximating higher-order task affinity matrix.
Recall this definition from equation \eqref{eq_higher_order_affinity}, Example \ref{ex_ho}.
We set $m = 2000$ so that the higher-order task affinity matrix converges while setting the subset size as $\alpha = 10$ (further ablation study will be provided in Section \ref{sec_ab}).

Using $M=1$ and $d=200$, our algorithm approximates $T^{\star}$  within \textbf{3.5}\% distance while using less than 1\% cost of computing $T^{\star}$. Further increasing $M$ to $5$, the distance drops to \textbf{2.7}\%. Again, the computation cost is only 3\% of computing $T^{\star}$.
This takes 11.9 GPU hours and uses \textbf{32.8}$\times$ less computation than fully-trained models.

\subsubsection{Accelerating task affinity computation on text and image data sets.} We have shown that \algo{} significantly reduces the computational cost in task affinity estimation. To verify that these efficiency gains are consistent across different data modalities, we apply \algo{} to a text classification dataset (RTE) and an image classification dataset (DomainNet) \cite{peng2019moment}. The RTE data set contains 100 tasks. We use T5-Base and compute higher-order task affinity with 2000 subsets of size 10. The DomainNet data set contains six tasks. We use ResNet-50 and compute higher-order task affinity with 20 subsets of size 3.
On the two data sets, our algorithm reduces computation by 42.6$\times$ and 9.5$\times$, respectively, compared to computing true higher-order task affinities while incurring less than 3\% relative error. 
The smaller speedup in the image dataset is due to the fewer total models trained on task subsets.

\subsubsection{Scaling task affinity estimation to very large instances}
Lastly, we estimate task affinities on the Orkut graph by varying $n$ from 100 to 500. We measure the distance between the estimated and the true pairwise affinity by downsampling the number of pairs to 2000. Figure \ref{fig_scaling} shows the comparison. 
We observe that our algorithm scales to as many as 500 tasks, using only $112.3$ GPU hours, which is much faster than computing $T^{\star}$.
Moreover, the relative distance to the true scores remains within 5\%.

\begin{figure*}[t!]
    \centering
    \begin{subfigure}[b]{0.49\textwidth}
    \begin{minipage}[b]{0.49\textwidth}
        \centering
        \includegraphics[width=\textwidth]{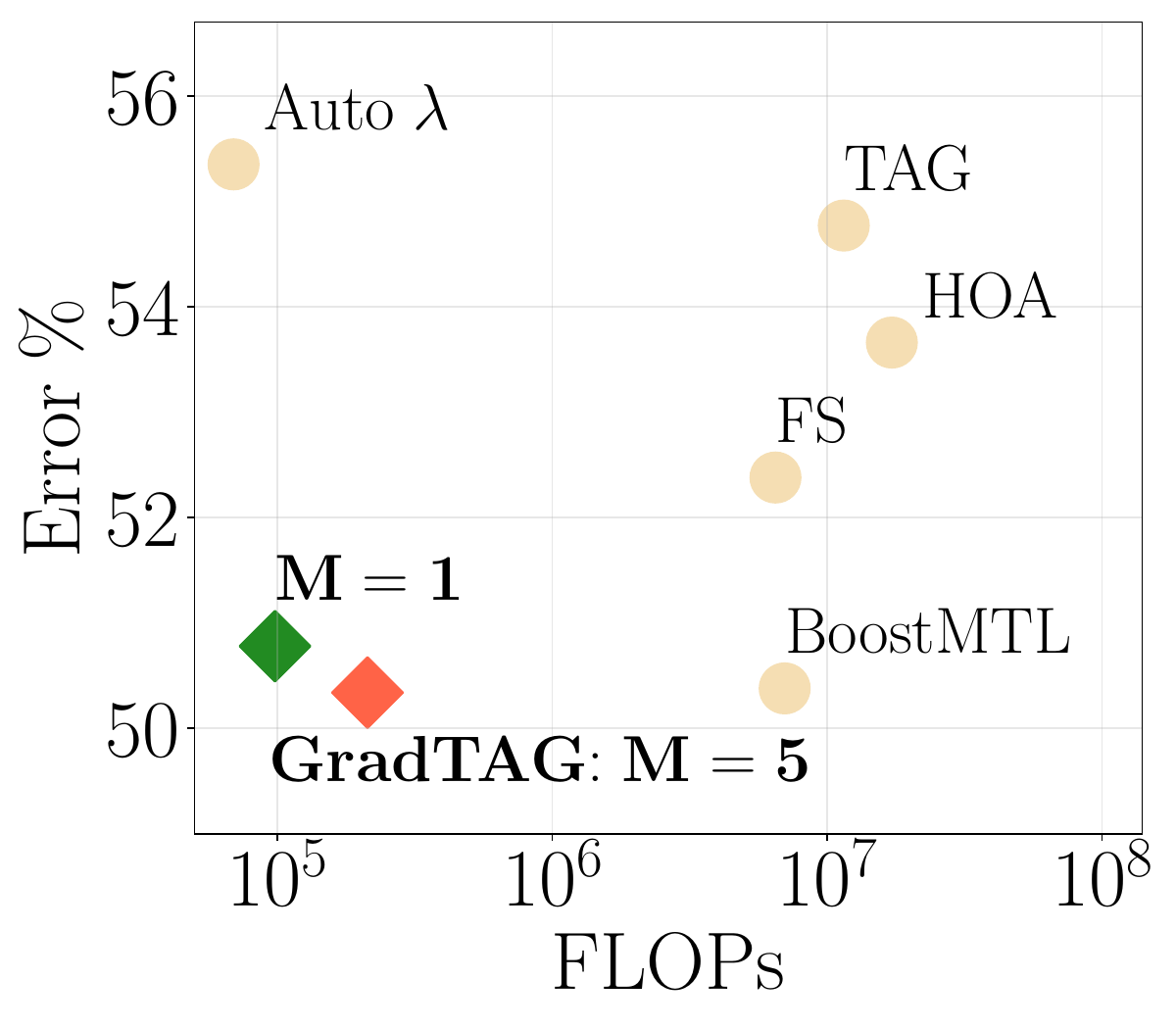}
    \end{minipage}
    \begin{minipage}[b]{0.49\textwidth}
        \centering
        \includegraphics[width=\textwidth]{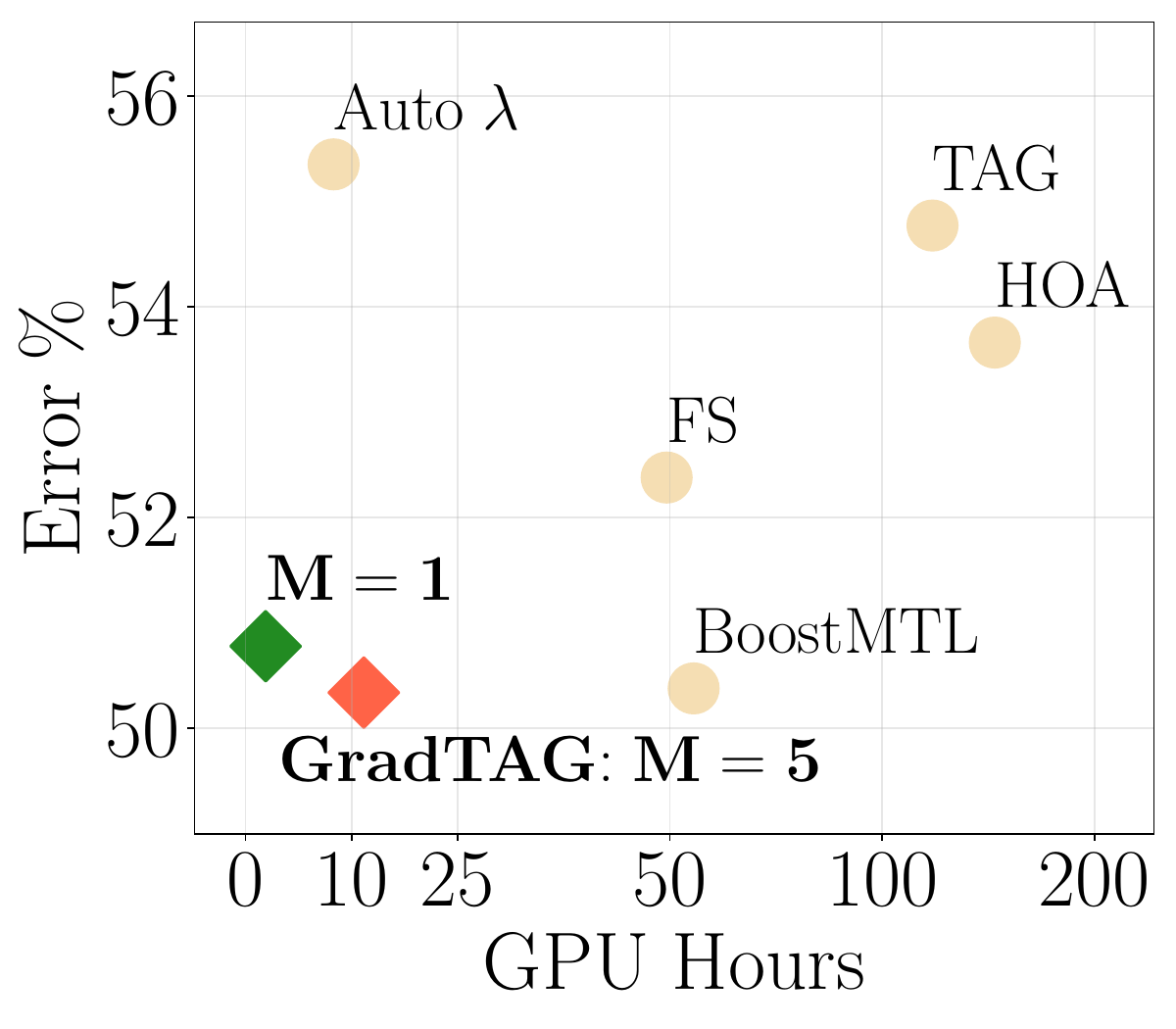}
    \end{minipage}
    \subcaption{Multi-label classification on graphs (The YouTube network)}\label{fig_compare_runtime_performance_community}
    \end{subfigure}
    \hfill
    \begin{subfigure}[b]{0.49\textwidth}
    \begin{minipage}[b]{0.49\textwidth}
        \centering
        \includegraphics[width=\textwidth]{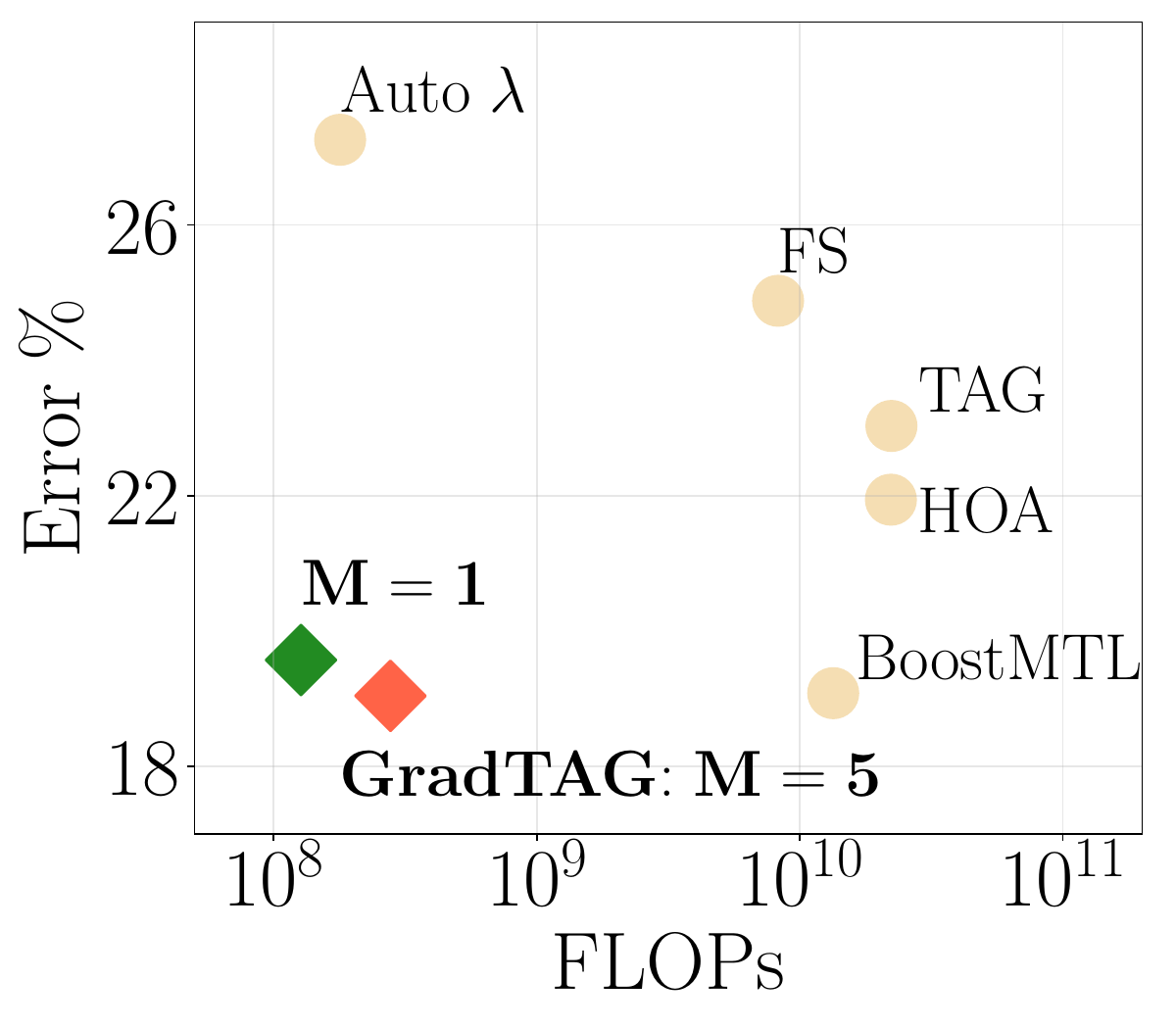}
    \end{minipage}
    \begin{minipage}[b]{0.49\textwidth}
        \centering
        \includegraphics[width=\textwidth]{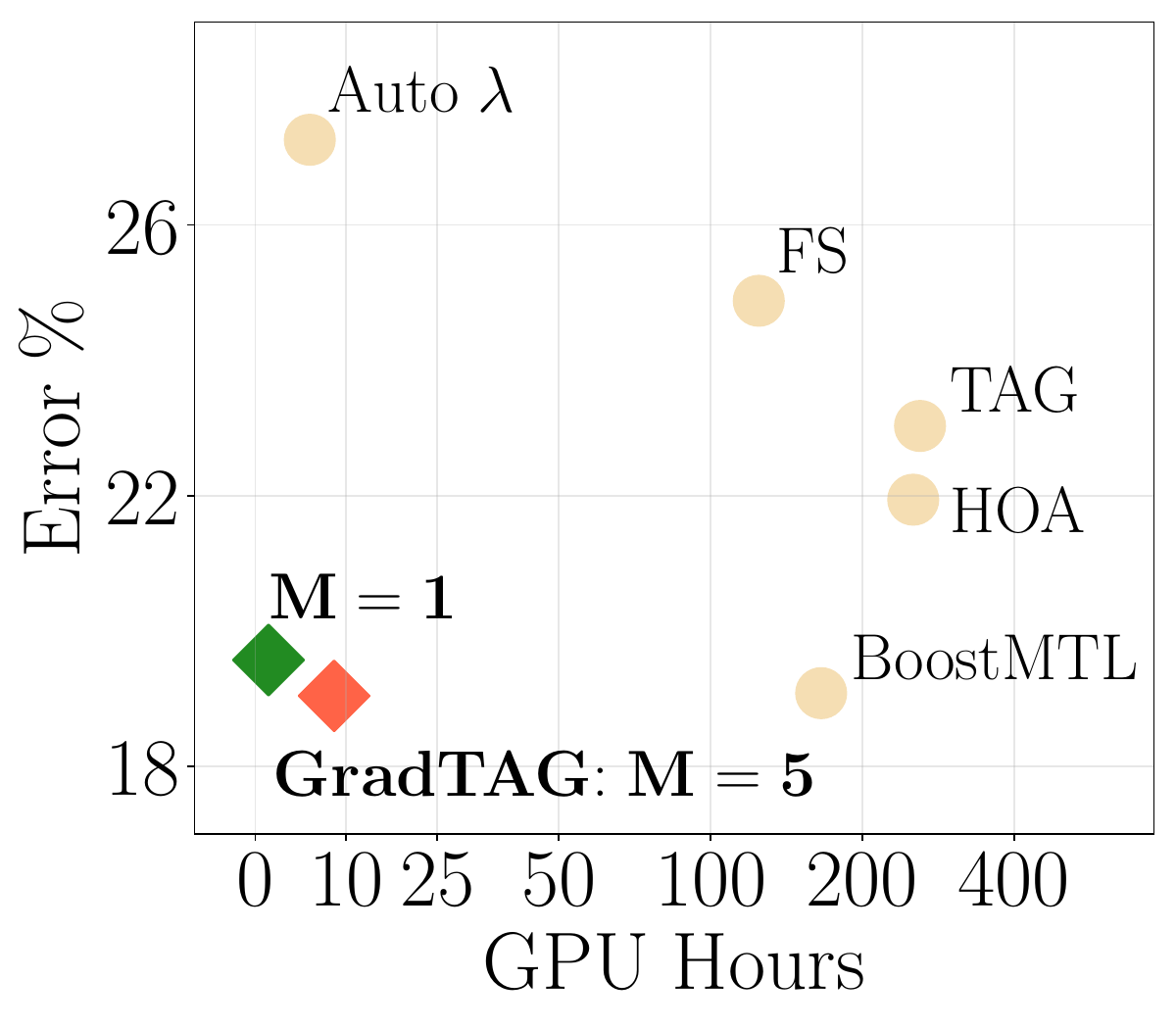}
    \end{minipage}
    \subcaption{Instruction fine-tuning of language models (On the RTE dataset)}\label{fig_compare_runtime_performance_finetuning}
    \end{subfigure}
    \caption{This figure illustrates the tradeoff between error rate and computation cost, measured by the number of FLOPs and GPU hours. Compared to multitask learning baselines, our approach achieves the Pareto optimal balance between error rate and computation cost. Recall that $M$ is the number of meta-initializations used in \acronym. The number of FLOPs is reported in the Giga FLOPs unit. For both settings, there are $n=100$ tasks.
    Our approach delivers comparable test accuracy to all baselines, using 32.8$\times$ fewer FLOPs and 5.2$\times$ less GPU hours than all baselines.}
\end{figure*}

\subsection{Comparison for Task Grouping}\label{sec_exp_tg}

\subsubsection{Baselines.} 
We set up many baselines covering heuristic solutions and recent optimization techniques.

{\it Forward Selection (FS) and Backward Selection (BS)} \cite{hastie2009elements}: These are standard approaches to perform subset selection, and we adapt them to task selection.

{\it Higher-Order Approximation (HOA)} \cite{standley2020tasks}: This algorithm computes pairwise task affinities between every two tasks and averages them to approximate higher-order affinities. It uses a branch-and-bound search algorithm to identify task groupings.

{\it Task Affinity Grouping (TAG)} \cite{fifty2021efficiently}: This approach computes the task affinity by evaluating the projecting one task’s gradients onto another task’s gradients during training. TAG also uses the branch-and-bound search algorithm to identify grouping.

{\it Auto-$\lambda$} \cite{liu2022auto}: This bilevel optimization technique balances the ratio of each task relative to the average objective of all tasks.

{\it BoostMTL} \cite{li2023boosting}: This approach computes higher-order task affinity between two tasks as the prediction loss of one task jointly trained with another task and a random subset of the remaining tasks, followed by spectral clustering to identify task groupings.

\subsubsection{Multi-label classification on graphs}

We report the result from applying our algorithm to overlapped community detection. We use our algorithm to estimate higher-order task affinity scores and then cluster the tasks.
We illustrate our results in Figure \ref{fig_compare_runtime_performance_community}, while deferring a full comparison to Appendix \ref{sec_experiment_details}. We use $1 - $ Macro $F_1$-score as the error rate on multi-label classification datasets. 
First, we confirm that our algorithm outperforms single-task learning that trains one model on each task by 2.1\% (as also evidenced by prior works on multitask learning \cite{zhang2021survey}).

We note that our algorithm reduces the error rate compared to all baselines while using \textbf{32.8}$\times$ fewer FLOPs and \textbf{5.2}$\times$ fewer GPU hours compared to the closest baseline, with $M = 5$

We can set $M=1$ for further speed up. This results in \textbf{71.4}$\times$ fewer FLOPs and \textbf{26.2}$\times$ less GPU hours than the closest baseline. The decrease in performance is only 0.3\%.%

\subsubsection{Fine-tuning language models} Next, we report the results from fine-tuning language models (T5 base) on text classification with $n = 100$ instructions. Again, we use our algorithm to estimate higher-order task affinity scores and apply SDP clustering to group tasks. We illustrate our results in Figure \ref{fig_compare_runtime_performance_finetuning} while deferring the complete comparison to Appendix \ref{sec_experiment_details}. We use $1 - $ accuracy as the error rate on the text classification datasets.
In particular, our algorithm outperforms single-task learning by 1.9\%. 

With $M=5$, our algorithm shows comparable performance to all baselines while using \textbf{48.2}$\times$ fewer FLOPs and \textbf{10.6}$\times$ less GPU hours.
By reducing $M$ to $1$, our algorithm further uses \textbf{105.4}$\times$ less FLOPs and \textbf{53.2}$\times$ less GPU hours, with only 0.5\% performance decrease.

\subsubsection{Discussion of clustering algorithms and hyper-parameters.}\label{sec_ab}

We discuss the design choices of Algorithm \ref{alg_approximate_grouping}. 
First, we study the SDP-based clustering vs. spectral and Lloyd's clustering. Across six datasets, SDP-based clustering outperforms these classical algorithms by an average of 1.2\%.
Next, we discuss the number of clusters $k$ and the rounding threshold $\lambda$.
We vary $k$ between $5, 10, 20,$ and $40$ (recall that $n = 100$). We note that the performance stabilizes when $k = 20$. Thus, we set $k = 20$.
For $\lambda$, we choose between $\frac 1 n$ and $\frac {10} n$, and select the value that results in $k$ clusters. 

Recall that Algorithm \ref{alg_approximate_grouping} also requires setting the number of subsets $m$ and each subset's size $\alpha$. Given $n=100$, we vary $m$ from $1000$ to $3000$ and observe that the result stabilizes when $m$ reaches $2000$. Thus, we set $m = 2000$. For $\alpha$, we choose it between $5, 10,$ and $20$. We choose $\alpha = 10$, as it yields better results than the rest.

\section{Conclusion}

This paper designs an efficient estimation algorithm to compute task affinity scores.
The main idea is to pre-train a meta-initialization on all tasks and then use the initialization's gradients to estimate the fine-tuned model parameters for a particular task combination using logistic regression.
A random projection is applied to the gradients to reduce the dimension of the regression.
Then, we design a robust clustering algorithm to accompany the task affinity estimation, which together yields an efficient multitask learning algorithm.
Experiments show that the algorithm can scale to as many as $500$ tasks on large graphs while accurately approximating the true task affinity scores.
The overall algorithm gives the best tradeoff between computation and performance compared to existing multitask learning methods.

We discuss several aspects of future work. First, it would be interesting to design novel dimension reduction and clustering methods in \acronym{}, and they will likely depend on downstream applications.
Second, it would be interesting to see if boosting could be used in branching neural networks, another type of multitasking architecture that trains a joint model on all tasks. A naive application of our method to group at the layer level is to start with a joint model and gradually split layers into task groups from input to output.
In each layer, the estimation procedure (based on layer-level features) may be used to compute task affinity scores and then group them accordingly.
This would help reduce the final model to a single neural network.

\section*{Acknowledgement}

Thanks to the anonymous referees for their comments.
This research is partly supported by NSF award IIS-2412008 and Northeastern University's Transforming Interdisciplinary Experiential Research (TIER) 1: Seed Grant/Proof of Concept Program.
D. Li was also partially funded by JPMorgan Chase \& Co. Any views or opinions expressed herein are solely those of the authors listed, and may differ from the views and opinions expressed by JPMorgan Chase \& Co. or its Affiliates.

\begin{refcontext}[sorting=nyt]
\balance
\printbibliography
\end{refcontext}

\onecolumn
\appendix

\section{Proof of Proposition \ref{prop_jl}}\label{proof_prop}

For this proof, we shall focus on binary classification. As discussed in Remark \ref{rem_ext}, the extension to multiple classifications requires additional notations, but the proof is straightforward.

\begin{proof}[Proof of Proposition \ref{prop_jl}]
    Recall that we define the minimizer for the logistic regression after random projection as $\hat W_d$.
    To make it clear, we annotate the vector with its dimension so that it is easy to distinguish.
    $\hat W_d$ is the minimizer of the following problem
    \begin{align}
        \min h_1(W) = \frac 1 {n_{\cS}} \sum_{i=1}^{n_{\cS}} \log \left(  1 + \exp\left ( - y_i g_i^{\top} P W + b_i \right)\right), \text{ for } W \in \real^d, \label{eq_hatWd}
    \end{align}
    where we recall that $P$ is a $p$ by $d$ random projection matrix,
    $g_i = \nabla_W f_{\theta^{\star}}(x_i, y_i)$,
    and $b_i = - y_i f_{\theta_\star}(x_i, y_i)$.
    
    Now, we define an intermediate solution $\overline W_p$ as follows
    \begin{align}
        \min h_2(W) = \frac 1 {n_{\cS}} \sum_{i=1}^{n_{\cS}} \log \left( 1 + \exp\left ( - y_i g_i^{\top} P P^{\top} (W - \theta^{\star}) + b_i \right) \right). \label{eq_wbar}
    \end{align}
    We can see that the function value of $\hat W_d$ for equation \eqref{eq_hatWd} must be less than the function value of $\overline W_p$ for equation \eqref{eq_wbar}.
    This is because the latter is a special case of the former.
    Thus, we first have that
    \begin{align}
        h_1(\hat W_d) \le h_2(\overline W_p). \label{eq_h1_h2}
    \end{align}
    
    Next, we compare $h_2(\overline W_p)$ with $\hat L(W^{\star})$.
    Recall that $W^{\star}$ is the minimizer for the following problem:
    \begin{align}\label{eq_star}
        \min \hat L(S) = \frac 1 {n_{\cS}} \sum_{i=1}^{n_{\cS}} \log\left( 1 + \exp\left(- y_i f_W(x_i, y_i) \right) \right).
    \end{align}

    We note that there are two sources of errors in this comparison.
    The first is the error between $f_W(x_i, y_i)$ and its Taylor's expansion $g_i^{\top}(W - \theta^{\star}) + b_i$.
    The second is the error introduced by the random projection.
    
    To make it easier to compare between equation \eqref{eq_star} with \eqref{eq_wbar}, let us expand the former as follows:
    \begin{align}
        & \min \frac 1 {n_{\cS}} \sum_{i=1}^{n_{\cS}} \log\left( 1 + \exp\left( - y_i f_W(x_i, y_i) \right) \right) \\
        =& \min \frac 1 {n_{\cS}} \sum_{i=1}^{n_{\cS}} \log \left( 1 + \exp\left( - y_i (b_i + g_i^{\top} (W - \theta^{\star}) + \epsilon_i) \right)\right) \tag{we use $\epsilon_i$ to denote Taylor's expansion error for $x_i, y_i$} \\
        =& \min \frac 1 {n_{\cS}} \sum_{i=1}^{n_{\cS}} \log\left( 1 + \exp\left( - y_i g_i^{\top} \left( (PP^{\top} + (\id - PP^{\top}) \right) (W - \theta^{\star}) + b_i \right) \right) \label{eq_nS1}
    \end{align}
    Let us denote
    \begin{align}
        \tilde \epsilon_i = g_i^{\top} (\id - PP^{\top}) (W - \theta^{\star}).
    \end{align}
    
    Thus, we can see that the difference between $W^{\star}$ and $\hat W$ can be attributed to the error term $\tilde \epsilon_i$.
    We rewrite equation \eqref{eq_nS1} as follows to make it clear
    \begin{align}
        \min \frac 1 {n_{\cS}}\sum_{i=1}^{n_{\cS}} \log\left( 1 + \exp(- y_i g_i^{\top} PP^{\top} (W - \theta^{\star}) + b_i + \tilde \epsilon_i) \right).
    \end{align}
    
    Now we bound the magnitude of $\tilde \epsilon_i$.
    Our idea is to use the fact that the logistic loss is $1$-Lipschitz continuous (to see that, one just needs to verify that
    \begin{align*} \abs{\log(1 + \exp(-x)) - \log(1 + \exp(-y))} \le \abs{x - y}. \end{align*}
    
    With this, we could then show that $h_2(\overline W_p)$ and $\hat L(W^{\star})$ are relatively close to each other.
    By definition, $h_2(\overline W_p) \le h_2(W^{\star})$.
    Additionally,
    \begin{align}
        & \abs{h_2(W^{\star}) - \hat L(W^{\star})} \\
        =& \frac 1 {n_{\cS}} \sum_{i=1}^{n_{\cS}} \abs{\log\left(1 + \exp(- y_i g_i^{\top} PP^{\top} (W - \theta^{\star})) + b_i\right) - \log\left(1 + \exp(- y_i g_i^{\top} PP^{\top} (W - \theta^{\star}) + b_i + \tilde \epsilon_i)\right)} \\
        \le & \frac 1 {n_{\cS}}\sum_{i=1}^{n_{\cS}} \abs{\tilde \epsilon_i}. \label{eq_nS_ts}
    \end{align}
    
    Recall from the assumption that the averaged Taylor's expansion error is at most $\delta$.
    Thus,
    \[ \frac 1 {n_{\cS}} \sum_{i=1}^{n_{\cS}} \abs{b_i} \le \delta. \]
    
    Next, by the Johnson-Lindenstrauss transformation \cite{johnson1984extensions} (For a modern exposition, see, e.g., lectures notes by Gregory Valiant: \url{https://theory.stanford.edu/~valiant/teaching/CS265/lectureNotes/l9.pdf}), provided that $d = O\big(\frac{\log p}{\epsilon^2}\big)$, we have
    \begin{align*}
        \abs{\inner{g_i}{W - \theta^{\star}} - \inner{P g_i}{P (W  - \theta^{\star})}} 
        \le \epsilon \bigabs{\inner{g_i}{W - \theta^{\star}}}
        \le {2 G D }{\epsilon}.
    \end{align*}

    Thus, applying the above two steps back into equation \eqref{eq_nS_ts}, we can now conclude that
    \begin{align}
        \abs{h_2(W^{\star}) - \hat L(W^{\star})}
        \le \delta + {2 G D }{\epsilon}. \label{eq_h2_hatL}
    \end{align}
    Applying equation \eqref{eq_nS_ts} back into equation \eqref{eq_h1_h2}, we can now conclude that
    \begin{align}
        h_1(\hat W_d) \le h_2(\overline W_p) \le h_2(W^{\star})
        \le \hat L(W^{\star}) + \delta + {2G D }{\epsilon}.\label{eq_final1}
    \end{align}

    To finish the proof, we can apply the above calculation to compare between $h_1(\hat W_d)$ and $\hat L(P \hat W_d +  \theta^{\star}))$, to get that
    \begin{align}
        \abs{h_1(\hat W_d) - \hat L(P \hat W_d + \theta^{\star}))}
        \le \delta + {2 GD } {\epsilon}.\label{eq_final2}
    \end{align}
    Combining equations \eqref{eq_final1} and \eqref{eq_final2} together, we finally conclude that
    \begin{align}
        \hat L(P \hat W_d + \theta^{\star})
        \le \hat L(W_p^{\star}) + 2\delta + {4 G D }{\epsilon}.
    \end{align}
    This completes the proof of Proposition \ref{prop_jl}.
\end{proof}

It would also be interesting to examine Taylor's expansion up to the Hessian in equation \eqref{eq_taylor}.
This requires additional computation of Hessian vector products.
After that, one needs to solve a quadratic program that depends on the Hessian matrix. This is left for future work.

Lastly, there is a line of work on model agnostic meta-learning and continual learning (See, e.g., survey article by \citet{hospedales2021meta}).
It would be interesting to see if our method can be applied to this setting (i.e. estimating fine-tuned model parameters without backpropagation). This is a promising direction for future work.

\section{Data Matrix for Example \ref{ex_planted}}

For completeness, we report the data matrix $T$ used to generate the clusters in Example \ref{ex_planted}.

\begin{align*}
    T = \left[\begin{array}{cccccc}
            7 & 7 & 6 & 6 & 5 & 5 \\
            7 & 7 & 6 & 6 & 5 & 5 \\
            6 & 6 & 20 & 20 & 19 & 19 \\       
            6 & 6 & 20 & 20 & 19 & 19 \\              
            5 & 5 & 19 & 19 & 20 & 20 \\                     
            5 & 5 & 19 & 19 & 20 & 20 \\                            
    \end{array}\right].
\end{align*}

\section{Additional Experiments}\label{sec_experiment_details}

\subsection{Implementations} 

\subsubsection{Models.} We use the SIGN model \citep{frasca2020sign} as the encoder in the multitask learning models on the community detection tasks. 
The encoder involves three layers, each with a fixed width of 256 neurons. Our choice of this encoder is without loss of generality, and our observations also apply to other encoders. We construct the node features from the VERSE embedding \cite{tsitsulin2018verse}, which encodes personalized PageRank vectors known as useful for community detection. We use the same number of model parameters for the Auto-$\lambda$ and MoE baselines as for the other task grouping baselines. 

On text classification tasks, we use T5-Base as the base model. We use LoRA fine-tuning \cite{hu2022lora}, which is a parameter-efficient fine-tuning method. 
For each dataset, we evaluate the average performance over all $100$ instructions. In our approach, we view one instruction as one task. We train the model with the AdamW optimizer with a learning rate of $5 \times 10^{-5}$ for 5,000 gradient update steps. We vary the rank of LoRA between 4, 8, 16, 32, 64, and 128. We find that a rank of $4$ leads to the best performance; thus, we set the rank as 4 in our experiments. 

\subsubsection{Baselines.} We describe the details of Forward selection: Start from all empty groups. Enumerate through all tasks by adding one task to one of the existing groups which results in the best average performance. 
In Backward selection, we start from a group with all tasks and other groups as empty. Enumerate through all tasks by removing one task from the first group and assigning the task to the group which results in the best average performance. 

To be representative in terms of relative improvement, we also compare the performance with conventional methods for community detection, including BigClam \cite{yang2013overlapping}, Louvain clustering \cite{blondel2008fast},  Network embedding methods including Node2Vec \cite{grover2016node2vec}, and VERSE \cite{tsitsulin2018verse}, and GNN-based community detection methods including MinCutPool \cite{bianchi2020spectral} and Deep Modularity Networks \cite{tsitsulin2023graph}. We noted that our approach outperforms the community detection baselines. The comparison results are reported in Table \ref{tab_community_detection}.

\subsection{Omitted results}

\subsubsection{Additional task grouping results.}  We illustrate the tradeoff between the error rate and the computation cost in terms of FLOPs and GPU hours of the other four datasets in our experiments in Figure \ref{fig_compare_runtime_performance_2}. We observe that our approach, \acronym{}, consistently achieves Pareto optimal in the evaluation metrics. While achieving the comparable performance of the best baseline, our approach reduces the computation cost by  32.8$\times$ and 5.2$\times$ in terms of FLOPs and GPU hours, respectively. Compared to the baselines using the same level of computation cost, our approach improves the MTL performance over the baselines by 4\% on average. 

\begin{figure*}[!ht]
    \centering
    \begin{subfigure}[b]{0.48\textwidth}
    \begin{minipage}[b]{0.49\textwidth}
        \centering
        \includegraphics[width=\textwidth]{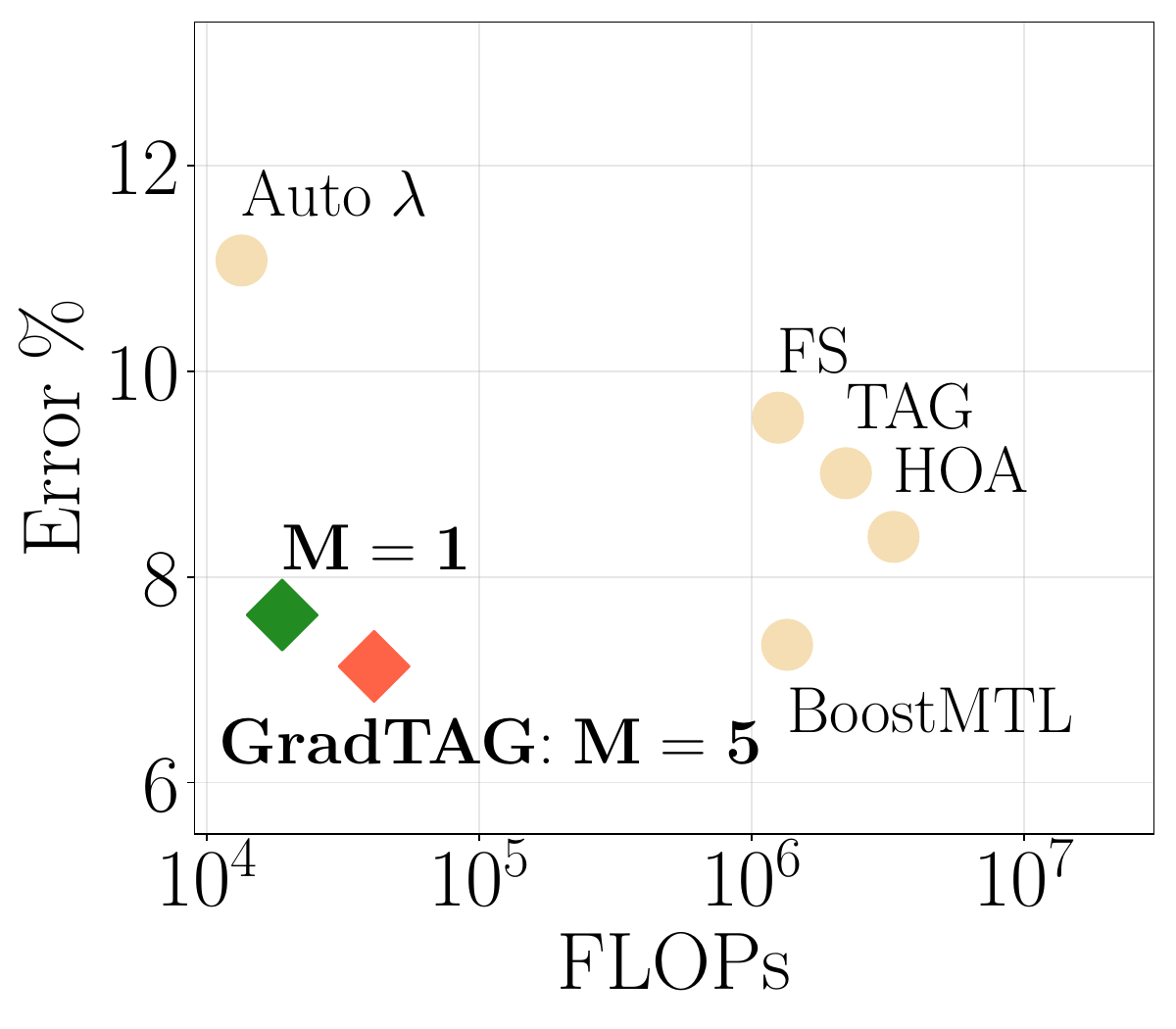}
    \end{minipage}
    \begin{minipage}[b]{0.49\textwidth}
        \centering
        \includegraphics[width=\textwidth]{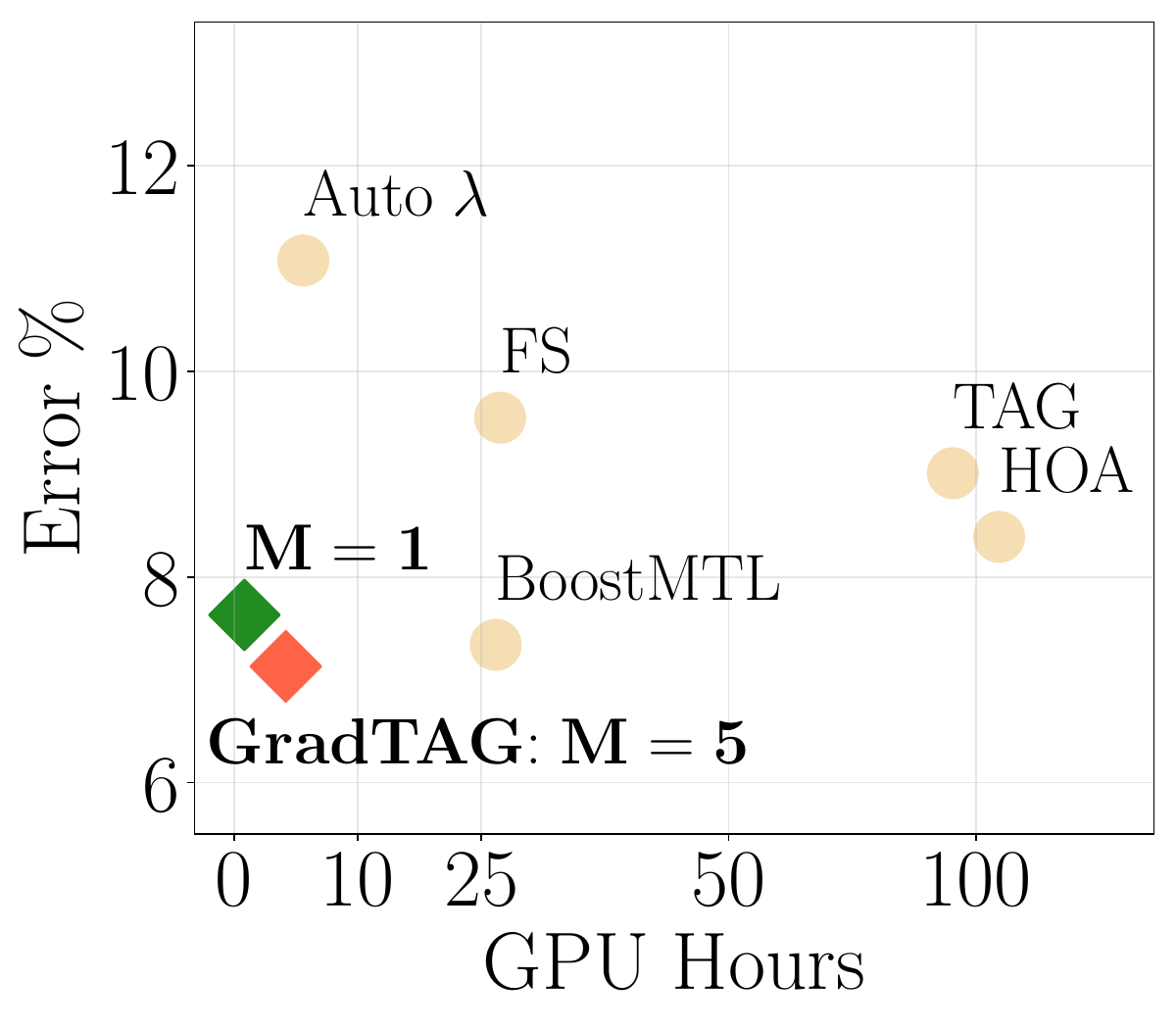}
    \end{minipage}
    \subcaption{Multi-label classification on graphs (The Amazon network)}
    \end{subfigure}\hfill
    \begin{subfigure}[b]{0.48\textwidth}
    \begin{minipage}[b]{0.49\textwidth}
        \centering
        \includegraphics[width=\textwidth]{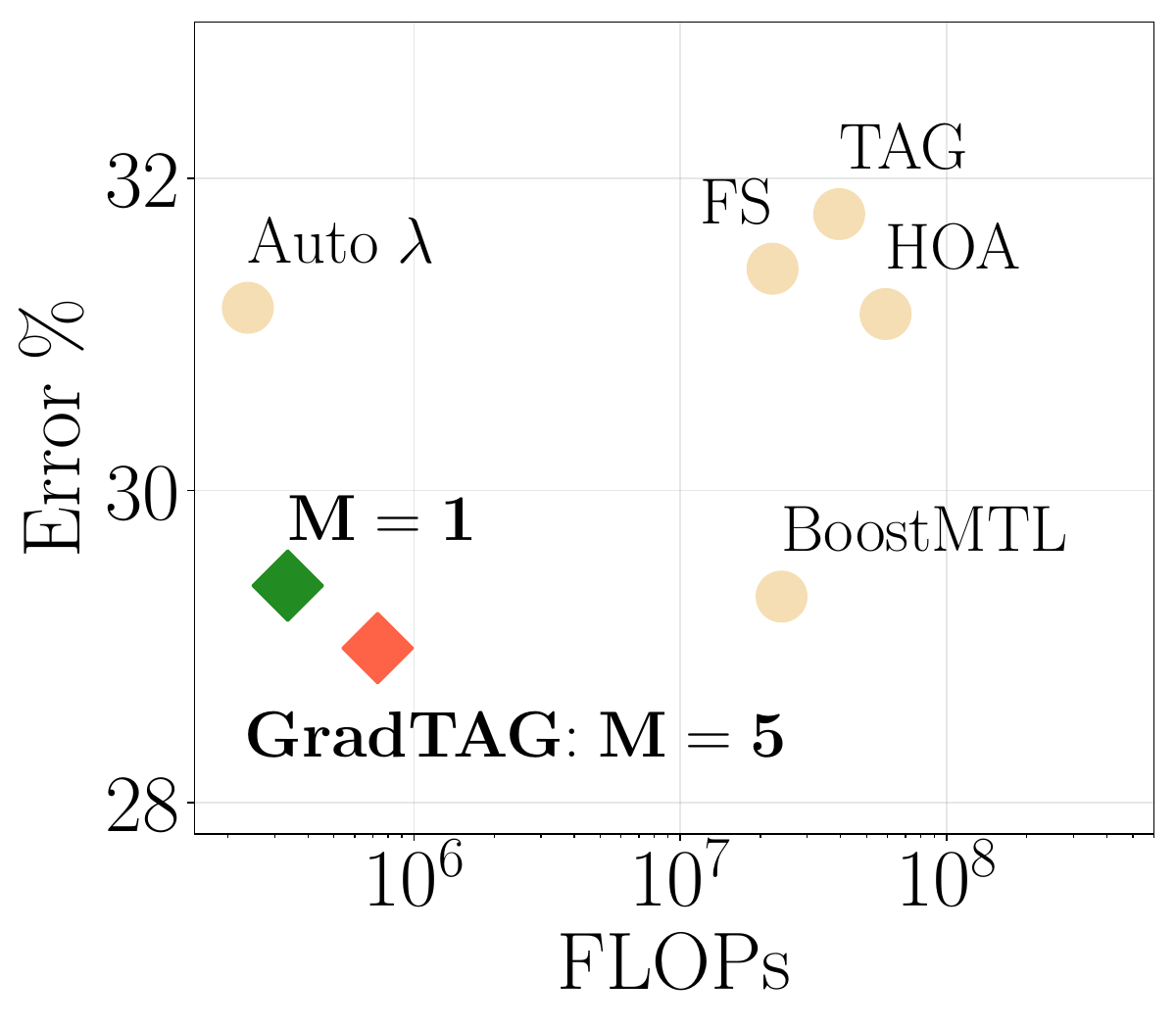}
    \end{minipage}
    \begin{minipage}[b]{0.49\textwidth}
        \centering
        \includegraphics[width=\textwidth]{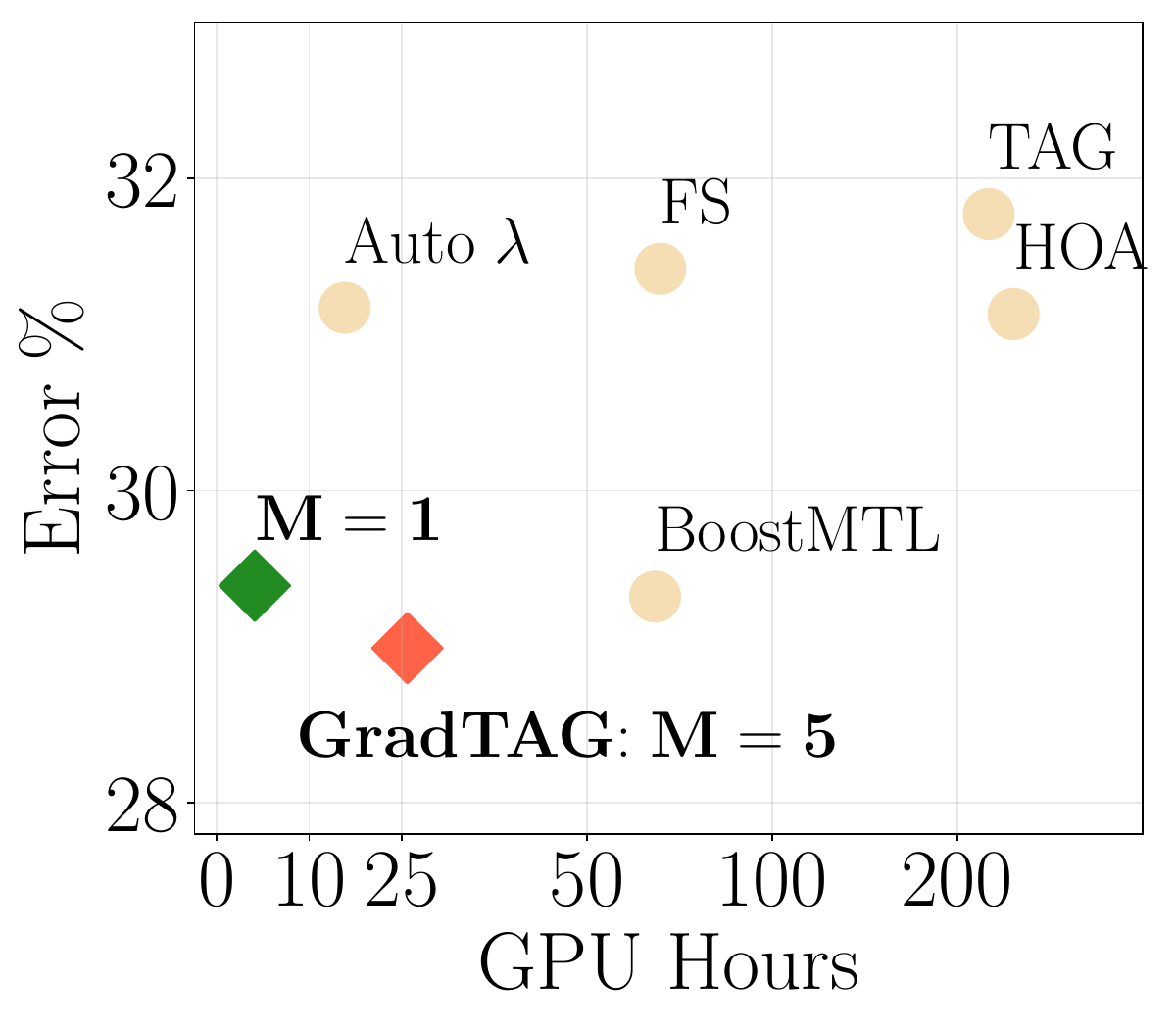}
    \end{minipage}
    \subcaption{Multi-label classification on graphs (The DBLP network)}
    \end{subfigure}
    \begin{subfigure}[b]{0.48\textwidth}
    \begin{minipage}[b]{0.49\textwidth}
        \centering
        \includegraphics[width=\textwidth]{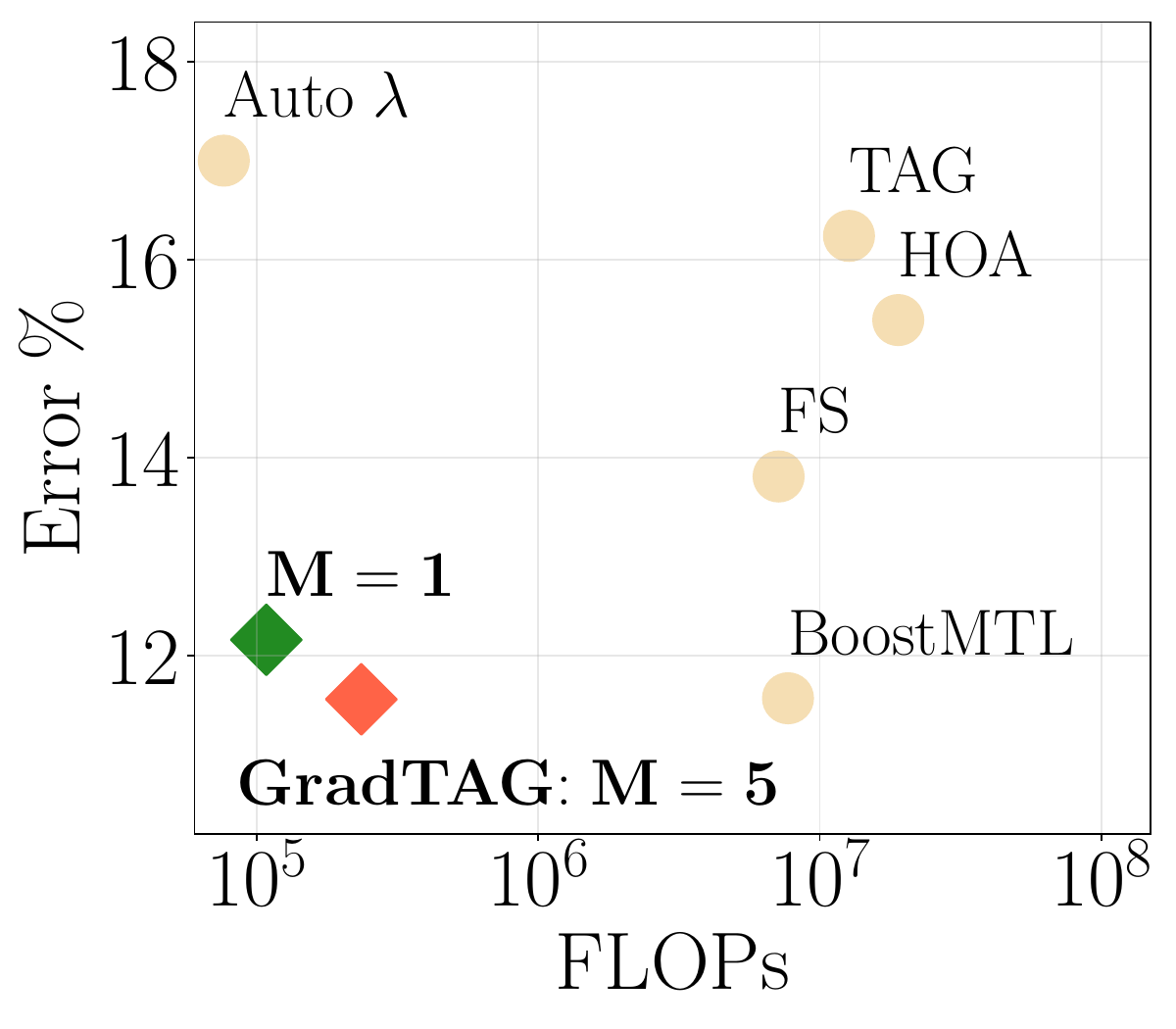}
    \end{minipage}
    \begin{minipage}[b]{0.49\textwidth}
        \centering
        \includegraphics[width=\textwidth]{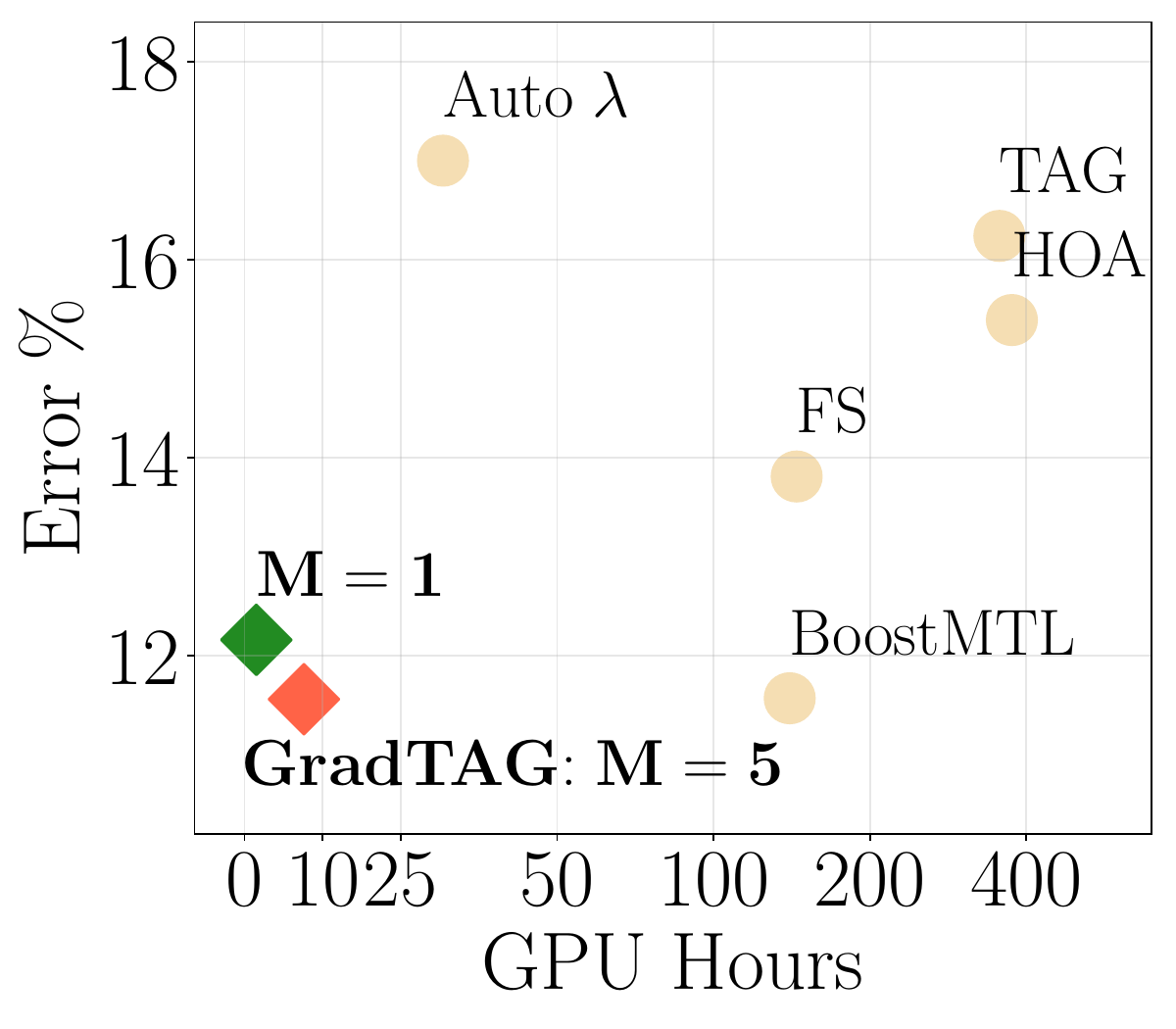}
    \end{minipage}
    \subcaption{Multi-label classification on graphs (The LiveJournal network)}
    \end{subfigure}\hfill
    \begin{subfigure}[b]{0.48\textwidth}
    \begin{minipage}[b]{0.49\textwidth}
        \centering
        \includegraphics[width=\textwidth]{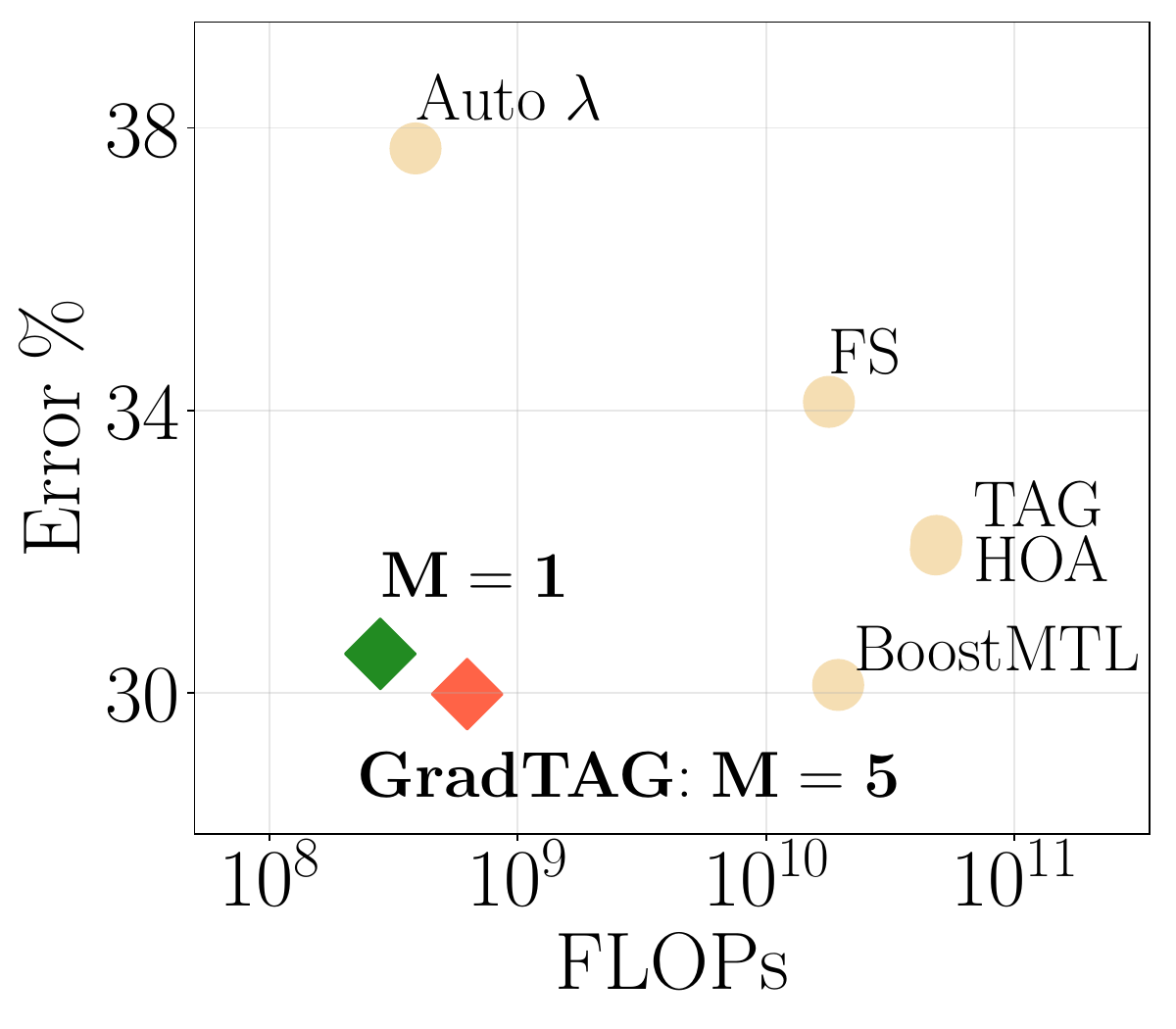}
    \end{minipage}
    \begin{minipage}[b]{0.49\textwidth}
        \centering
        \includegraphics[width=\textwidth]{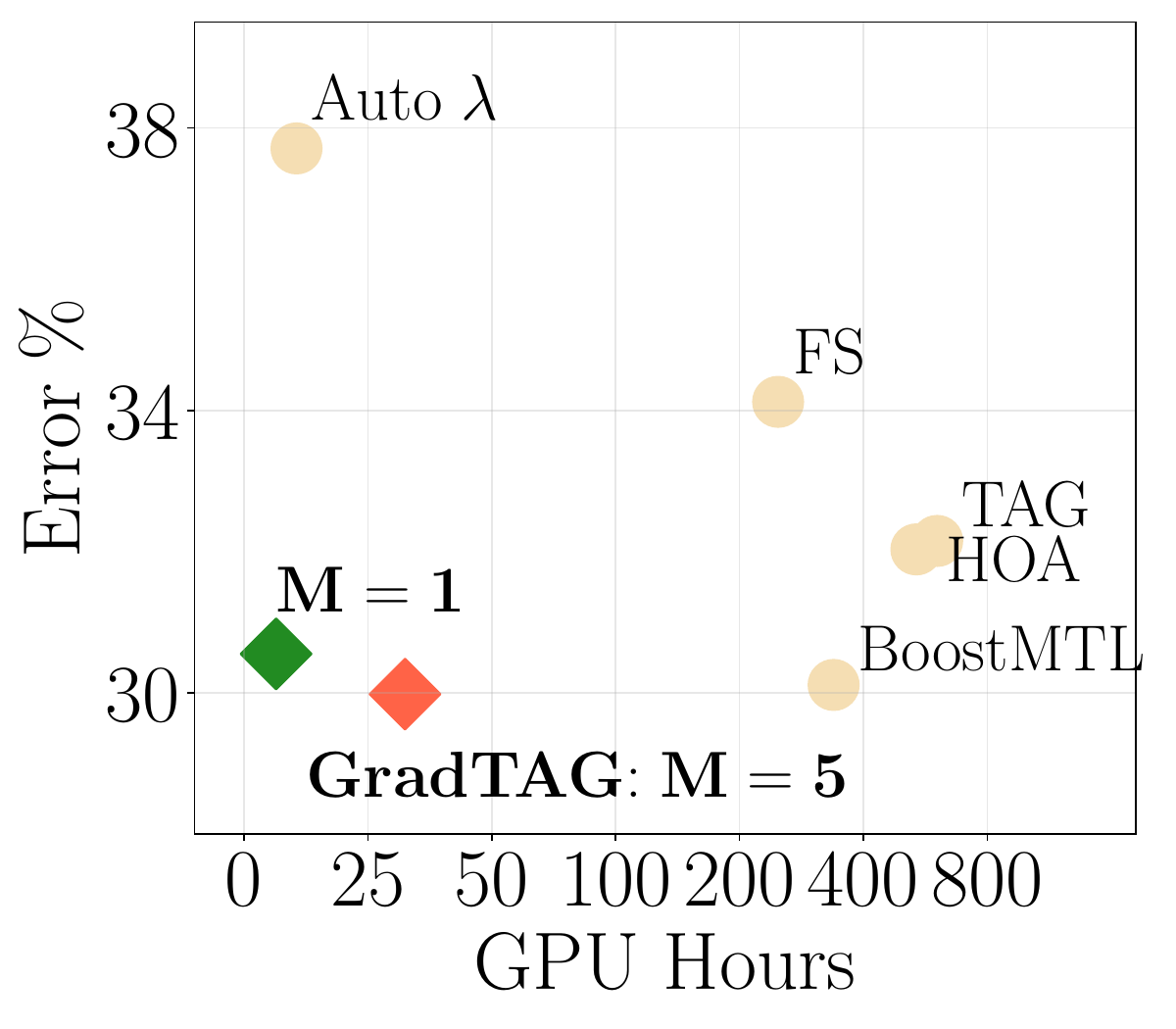}
    \end{minipage}
    \subcaption{Instruction fine-tuning of language models (On the WiC dataset)}
    \end{subfigure}
    \caption{This figure illustrates the tradeoff between the error rates and computation cost in terms of FLOPs and GPU hours on four datasets omitted in the main text. 
    Our approach, \acronym{}, consistently achieves the Pareto optimal, delivering comparable test accuracy to other MTL baselines and using 32.8$\times$ fewer FLOPs and 5.2$\times$ less GPU hours than other baselines. 
    $M$ denotes the number of meta-initializations used in our approach. 
    }\label{fig_compare_runtime_performance_2}
\end{figure*}

\subsubsection{Correlation Between Estimated Affinities and True Scores}

Our results show that task grouping with our estimated task affinities can achieve competitive performance with the previous method that uses the fully computed higher-order task affinities. To explain these results, we hypothesize that the estimated task affinities are highly correlated with the true task affinities, resulting in similar task groupings and, consequently, comparable performance.  We compute the Spearman correlation between the estimated and true task affinities corresponding to one task $j$, i.e., the correlation between $[T_{1, i}, \ldots, T_{n, i}]$ and $[T^\star_{1, i}, \ldots, T^\star_{n, i}]$. 
We evaluated on the YouTube network of 100 tasks. We show that using $M=1$ meta-initialization, the estimated task affinities have a \textbf{0.91} correlation with the true scores averaged over all tasks. With $M=5$, the estimated scores have a \textbf{0.96} correlation with true scores.

\subsection{Tables of Full Comparisons}\label{sec_compare}

Here, we report the complete results for Section \ref{sec_exp_tg}.

\begin{table*}[h!]
\centering
\caption{We report the Macro $F_1$-score, computation cost as FLOPs, and runtime as GPU hours, on community detection tasks using four social networks. We compare our approach with MTL optimization methods, feature subset selection methods, and graph embedding methods. For each experiment, we report the results averaged over three random seeds and include their standard deviations.}\label{tab_community_detection}
{
\begin{tabular}{@{}lcccccccc@{}}
\toprule
Dataset   & {Amazon} & {Youtube} & {DBLP} & {LiveJournal}  \\ 
Nodes & {3,225} & {16,751} & {57,368} & {18,433}  \\
Edges & {20,524} & {104,513} & {420,122} &{1,397,580}  \\
\midrule
\textbf{Macro $F_1$-score}: Community Detection Methods \\
\midrule
BigClam     & 27.30 $\pm$ 0.26 & 18.84 $\pm$ 0.18  & 13.46 $\pm$ 0.11  & 22.50 $\pm$ 0.31 \\
Louvain clustering  & 60.95 $\pm$ 0.19  & 29.03 $\pm$ 0.34  & 36.73 $\pm$ 0.34  & 64.08 $\pm$ 0.17 \\
Node2Vec    & 39.05 $\pm$ 0.10 & 32.44 $\pm$ 0.18 & 28.72 $\pm$ 0.10  & 50.40 $\pm$ 0.29 \\
VERSE       & 61.00 $\pm$ 0.32 & 38.17 $\pm$ 0.12 & 53.48 $\pm$ 0.24  & 58.71 $\pm$ 0.48 \\
MinCutPool  & 84.24 $\pm$ 0.19 & 44.28 $\pm$ 0.49 & 67.49 $\pm$ 0.96 & 81.87 $\pm$ 1.06\\
Deep Modularity Networks  & 83.30 $\pm$ 1.07 & 43.58 $\pm$ 0.77 & 66.32 $\pm$ 0.15 & 79.84 $\pm$ 0.80\\
\midrule
\textbf{Macro $F_1$-score} \\
\midrule
Single task learning & 92.26 $\pm$ 4.62 & 46.69 $\pm$ 2.44 & 68.32 $\pm$ 4.18 & 87.78 $\pm$ 4.04\\
Multi-Gate MoE  & 88.92 $\pm$ 6.65  & 44.65 $\pm$ 4.28  & 68.83 $\pm$ 4.06  & 83.08 $\pm$ 4.89 \\
Auto-$\lambda$  & 88.08 $\pm$ 4.04  & 44.42 $\pm$ 2.24  & 68.95 $\pm$ 2.28  & 83.56 $\pm$ 3.21 \\ 
Forward Selection  & 90.45 $\pm$ 3.63  & 47.62 $\pm$ 2.84  & 68.58 $\pm$ 2.96  & 86.19 $\pm$ 2.61 \\
Backward Selection & 90.41 $\pm$ 5.98  & 47.68 $\pm$ 3.01  &  68.63 $\pm$ 2.97 & 85.91 $\pm$ 4.18 \\
Task Affinity Grouping & 90.99 $\pm$ 4.06 & 45.23 $\pm$ 2.73 & 68.23 $\pm$ 3.24 & 83.76 $\pm$ 3.77 \\
Higher-Order Approximation  & 91.61 $\pm$ 3.86 & 46.34 $\pm$ 2.57 & 68.87 $\pm$ 2.23 & 84.61 $\pm$ 2.56   \\
BoostMTL & {92.66 $\pm$ 4.85} & {49.62 $\pm$ 2.26} &  {70.68 $\pm$ 2.65}  & {\textbf{88.43} $\pm$ 2.70}   \\ \midrule
\acronym{} ($M=1$) & 92.37 $\pm$ 3.26  & 49.22 $\pm$ 2.56  & 70.61 $\pm$ 2.21 & 87.84 $\pm$ 2.44 \\
\acronym{} ($M=5$) & \textbf{92.87} $\pm$ 3.80  & \textbf{49.66} $\pm$ 2.24  & \textbf{71.01} $\pm$ 2.45 & \textbf{88.44} $\pm$ 2.80 \\
\midrule
\textbf{\# Giga FLOPs} \\
\midrule
Single task learning & 6.43 $\times 10^4$ & 3.34 $\times 10^5$ & 1.14 $\times 10^6$ & 3.67 $\times 10^5$ \\
Multi-Gate MoE  & 1.34 $\times 10^4$ & 6.94 $\times 10^4$ & 2.38 $\times 10^5$ &	7.63 $\times 10^4$ \\
Auto-$\lambda$       & 2.40 $\times 10^4$ & 1.36 $\times 10^5$ & 4.58 $\times 10^5$ & 1.44 $\times 10^5$ \\ 
Forward Selection  & 1.25 $\times 10^6$ & 6.49 $\times 10^6$ & 2.22 $\times 10^7$ & 7.14 $\times 10^6$ \\
Backward Selection & 1.35 $\times 10^6$  & 7.01 $\times 10^6$ & 2.40 $\times 10^7$ & 7.71 $\times 10^6$ \\
Task Affinity Grouping & 2.22 $\times 10^6$ & 1.15 $\times 10^7$ & 3.95 $\times 10^7$ & 1.27 $\times 10^7$ \\
Higher-Order Approximation  & 3.32 $\times 10^6  $ & 1.72 $\times 10^7$ & 5.90 $\times 10^7$ & 1.90 $\times 10^7$ \\
BoostMTL & 1.35 $\times 10^6$ & 7.01 $\times 10^6$ & 2.40 $\times 10^7$ & 7.71 $\times 10^6$ \\ \midrule
\acronym{} ($M=1$) & \textbf{1.89} $\times 10^4$ & \textbf{9.82} $\times 10^4$ & \textbf{3.36} $\times 10^5$ & \textbf{1.08} $\times 10^5$ \\
\acronym{} ($M=5$) & 4.11 $\times 10^4$ & 2.13 $\times 10^5$ & 7.31 $\times 10^5$ & 2.35 $\times 10^5$ \\
\midrule
\textbf{GPU Hours} \\
\midrule
Single task learning & 3.08 H & 4.27 H & 12.88 H & 4.30 H \\
Mixture-of-Experts & 7.01 H & 10.39 H & 17.28 H & 31.74 H \\
Auto-$\lambda$ & 8.00 H & 10.27 H & 18.96 H & 33.35 H  \\
Forward Selection & 26.91 H & 49.61 H & 53.38 H & 88.28 H \\
Backward Selection & 37.90 H & 62.89 H & 69.39 H & 105.94 H \\
Task Affinity Grouping & 92.68 H & 199.99 H & 224.22 H & 305.69 H \\
Higher-Order Approximation & 87.35 H & 197.26 H & 207.56 H & 294.69 H \\
BoostMTL & 24.46 H &	52.79 H & 59.19 H & 87.17 H \\ \midrule
\acronym{} ($M=1$) & \textbf{1.05 H} & \textbf{2.39 H} & \textbf{5.15 H} & \textbf{1.90 H} \\
\acronym{} ($M=5$) & 5.24 H & 11.95 H & 25.75 H & 9.50 H \\
\bottomrule
\end{tabular}
}
\end{table*}

\begin{table*}[t!]
\centering
\caption{We report the accuracy scores on the development set averaged over all instructions, on two sentence classification tasks from SuperGLUE. We compare our approach with MTL optimization methods and feature subset selection methods. For each experiment, we report the results averaged over three random seeds and include their standard deviations.}\label{tab_instruction_tuning_results}
{
\begin{tabular}{@{}lcccccc@{}}
\toprule
Dataset   &  RTE & WiC \\ %
\midrule
\textbf{Accuracy} \\ 
\midrule
Single task learning & 79.44$\pm$1.33 & 68.93$\pm$1.49\\
Mixture-of-Experts & 73.06$\pm$1.61 & 62.17$\pm$1.68 \\ %
Auto-$\lambda$   & 72.74$\pm$2.40 & 62.29$\pm$2.93 \\ %
Forward Selection   & 75.12$\pm$1.26 & 65.88$\pm$2.19 \\ %
Backward Selection & 75.09$\pm$1.68 &  66.44$\pm$1.98 \\
Task Affinity Grouping & 76.97$\pm$1.83 & 67.85$\pm$1.31\\
Higher-Order Approximation & 78.06$\pm$2.14 & 67.97$\pm$ 1.35 \\
BoostMTL &  {80.92$\pm$1.85} & {69.89$\pm$0.87}  \\ %
\midrule 
\acronym{} ($M=1$) & {80.43$\pm$1.23} & 69.45$\pm$1.21\\
\acronym{} ($M=5$) & {\textbf{80.96}$\pm$1.61} & \textbf{70.02}$\pm$1.35\\
\midrule
\textbf{\# Giga FLOPs} \\ 
\midrule
Single task learning & 4.48 $\times 10^8$ & 89.73 $\times 10^8$ \\
Mixture-of-Experts  & 8.94 $\times 10^7$ & 1.94 $\times 10^8$\\ 
Auto-$\lambda$      & 1.79 $\times 10^8$ & 3.88 $\times 10^8$\\ 
Forward Selection   & 8.27 $\times 10^9$ & 1.79 $\times 10^{10}$ \\ 
Backward Selection  & 8.94 $\times 10^9$ & 1.94 $\times 10^{10}$\\
Task Affinity Grouping & 2.23 $\times 10^{10}$ & 4.85 $\times 10^{10}$\\
Higher-Order Approximation & 2.22 $\times 10^{10}$ & 4.82 $\times 10^{10}$\\
BoostMTL            & 1.34 $\times 10^{10}$ & {1.95} $\times 10^{10}$\\
\midrule 
\acronym{} ($M=1$) & \textbf{1.27 $\times 10^8$} & \textbf{2.80} $\times 10^{8}$\\
\acronym{} ($M=5$) & 2.78 $\times 10^8$ & 6.26 $\times 10^8$\\
\midrule
\textbf{GPU Hours} \\ 
\midrule
Single task learning & 19.4 H & 23.8 H \\
Mixture-of-Experts & 7.9 H & 10.6 H \\ 
Auto-$\lambda$   &  9.0 H & 10.9 H \\ 
Forward Selection   & 114.5 H & 247.8 H \\  
Backward Selection & 123.8 H & 267.9  H \\
Task Affinity Grouping & 309.6 H & 669.9  H \\
Higher-Order Approximation & 307.6 H & 665.7 H \\
BoostMTL &   186.2 H & 269.0 H \\
\midrule 
\acronym{} ($M=1$) & \textbf{3.5 H} & \textbf{6.49 H} \\
\acronym{} ($M=5$) & 17.5 H & 32.50 H \\
\bottomrule %
\end{tabular}
}
\end{table*}

\end{document}